%% file: main.tex
\documentclass{article}

\usepackage[utf8]{inputenc}
\usepackage{graphicx}
\usepackage{amsmath}
\usepackage{amssymb}
\usepackage{amsfonts,amsthm}
\usepackage{color}
\usepackage{caption}
\usepackage{subcaption}
\usepackage{xspace}
\usepackage{array}
\usepackage{geometry}
\usepackage{float}
\usepackage[colorlinks,citecolor=blue,bookmarks=false]{hyperref}
\usepackage{authblk}

\usepackage[acronym,shortcuts]{glossaries}  %
\usepackage{amsthm}
\usepackage{booktabs}  %
\usepackage{prettyref}
\usepackage{makecell}
\usepackage[dvipsnames]{xcolor}
\usepackage{capt-of}
\usepackage{footmisc}

\input{shortcuts}

\title{Learning Neural Light Transport}

\input{authors}

\newacronym{mse}{MSE}{mean squared error}
\newacronym{cnn}{CNN}{convolutional neural network}
\newacronym{gan}{GAN}{generative adversarial network}
\newacronym{vae}{VAE}{variational autoencoder}
\newacronym{mssim}{MSSIM}{mean structural similarity index}
\newacronym{fid}{FID}{Fr\'echet inception distance}

\newtheorem{lemma}{Lemma}
\makeglossaries

\begin{document}

\maketitle

\input{content/abstract}

\input{content/introduction}

\input{content/related}

\input{content/method}

\input{content/experiments}

\input{content/conclusion}

\appendix

\input{content/supplementary/abstract}
\input{content/supplementary/architectures}

\input{content/supplementary/data}

\input{content/supplementary/training}
\input{content/supplementary/results}

\clearpage

\bibliographystyle{plain}
\bibliography{bibliography.bib}

\end{document}

%% file: shortcuts.tex
\newcommand{\ba}{\mathbf{a}}\newcommand{\bA}{\mathbf{A}}

\newcommand{\bD}{\mathbf{D}}

\newcommand{\bff}{\mathbf{f}}\newcommand{\bF}{\mathbf{F}} %

\newcommand{\bI}{\mathbf{I}}

\newcommand{\bK}{\mathbf{K}}

\newcommand{\bp}{\mathbf{p}}\newcommand{\bP}{\mathbf{P}}
\newcommand{\bq}{\mathbf{q}}
\newcommand{\br}{\mathbf{r}}

\newcommand{\bT}{\mathbf{T}}
\newcommand{\bu}{\mathbf{u}}

\newcommand{\bX}{\mathbf{X}}

\newcommand{\cD}{\mathcal{D}}

\newcommand{\figref}[1]{Fig.~\ref{#1}}
\newcommand{\secref}[1]{Section~\ref{#1}}

\newcommand{\eqnref}[1]{Eq.~\eqref{#1}}
\newcommand{\tabref}[1]{Table~\ref{#1}}

\DeclareMathOperator*{\argmin}{argmin~}

\makeatletter
\DeclareRobustCommand\onedot{\futurelet\@let@token\@onedot}
\def\@onedot{\ifx\@let@token.\else.\null\fi\xspace}
\def\eg{e.g\onedot} 
\def\ie{i.e\onedot}

\def\vs{vs\onedot}
\def\wrt{wrt\onedot}

\def\etal{et~al\onedot}

\makeatother

\newcommand{\boldparagraph}[1]{\vspace{0.2cm}\noindent{\bf #1:}}

\definecolor{darkgreen}{rgb}{0,0.7,0}
\definecolor{darkblue}{rgb}{0,0,0.7}

%% file: authors.tex
\author{\normalsize{
Paul Sanzenbacher$^{1,2}$ \qquad Lars Mescheder$^{1,2,3}$ \qquad Andreas Geiger$^{1,2}$
}}

\date{\normalsize{
$^1$Max Planck Institute for Intelligent Systems Tübingen\\
$^2$University of Tübingen\qquad
$^3$Amazon Tübingen\footnote{This work was done prior to joining Amazon.}\\
\texttt{\{firstname.lastname\}@tue.mpg.de}
}}

%% file: content/abstract.tex
\begin{abstract}
In recent years, deep generative models have gained significance due to their ability to synthesize natural-looking images with applications ranging from virtual reality to data augmentation for training computer vision models. While existing models are able to faithfully learn the image distribution of the training set, they often lack controllability as they operate in 2D pixel space and do not model the physical image formation process. In this work, we investigate the importance of 3D reasoning for photorealistic rendering. We present an approach for learning light transport in static and dynamic 3D scenes using a neural network with the goal of predicting photorealistic images. In contrast to existing approaches that operate in the 2D image domain, our approach reasons in both 3D and 2D space, thus enabling global illumination effects and manipulation of 3D scene geometry. Experimentally, we find that our model is able to produce photorealistic renderings of static and dynamic scenes. Moreover, it compares favorably to baselines which combine path tracing and image denoising at the same computational budget.
\end{abstract}

%% file: content/introduction.tex
\section{Introduction}

Photorealistic rendering is a core problem in graphics and vision.
Algorithms which are able to reason about direct and indirect illumination of  a scene (\ie, global illumination) have become an essential building block for a wide range of applications such as gaming, virtual reality, movies and others.
With the advent of deep learning, synthetic data generation emerged as another important application~\cite{Alhaija2018IJCV,Gaidon2016CVPR,Shrivastava2017CVPR,Dosovitskiy2017CORL,Varol2017CVPR} with the potential to satisfy the notorious data hunger of modern deep learning systems.
However, as modern deep neural networks require large amounts of data, most existing approaches rely on approximate rendering techniques to accelerate training~\cite{Gaidon2016CVPR,Shrivastava2017CVPR,Dosovitskiy2017CORL,Varol2017CVPR}.
Training embodied agents (\ie, using reinforcement learning) poses even stronger demands \wrt simulation time \cite{Espeholt2018ICML}.

\begin{figure}[t]
	\centering
	\includegraphics[width=1\linewidth]{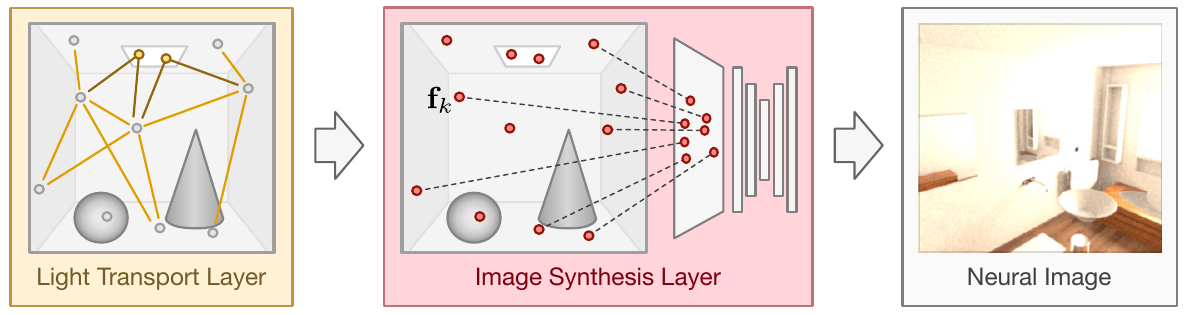}
	\caption{\textbf{Motivation.}
		We learn photorealistic rendering using a 3D Light Transport Layer in combination with a 2D Image Synthesis Layer.
		We demonstrate that our hybrid 3D-2D approach is able to synthesize realistic images with global illumination effects in real-time.
		\vspace{-1em}
	}
	\label{fig:motivation}
\end{figure}

Historically, photorealistic image synthesis is achieved using sampling-based rendering techniques~\cite{Pharr2016,Veach1998}
where the physics of light transport~\cite{Kajiya1986SIGGRAPH} are exploited to transform a physical description of a scene into a realistic image.
However, while physically based rendering yields photorealistic results, it is also notoriously slow with rendering times of up to multiple hours for a single image.

On the other hand, recent advances in deep learning enabled the generation of highly realistic images~\cite{Goodfellow2014NIPS,Mescheder2018ICML,Karras2018ICLR,Karras2019CVPR} in milliseconds on a commercial GPU.
Unfortunately,
most existing approaches make use of rather abstract latent representations which do not allow for precise control over the 3D content.
Moreover, the lack of a holistic scene description limits neural rendering approaches in their ability to render images that are consistent across viewpoints or time.
While some recent works~\cite{Sitzmann2019CVPR,Aliev2019ARXIV,Meshry2019CVPR,Noguchi2019ARXIV} have shown that neural network can produce consistent images for a given scene, these approaches usually 
do not explicitly reason about light transport.
Consequently, these approaches are not able to handle fine-grained geometric scene manipulation and do not integrate illumination into the 3D representation:
Consider a moving light source that is not visible in the current view. While it clearly influences illumination in the scene, an image-based approach can hardly reason about it.

\boldparagraph{Contributions} In this work, we investigate the importance of 3D \vs 2D reasoning for efficient learning-based photorealistic rendering.
Towards this goal, we present a learning-based approach (see \figref{fig:motivation} for a high-level overview) which can predict photorealistic images from a point-cloud based scene representation in real time.
In contrast to existing approaches, our method performs reasoning both in 3D and 2D space which allows for learning the physical light transport in a scene.
We hypothesize that this enables our method to handle scene modifications such as object translations, object removal and lighting changes.
At the same time, our method allows for learning useful heuristics (\eg, shadows that are not affected by moving objects) from the training data, enabling fast rendering without sacrificing quality.
We introduce two variants of our approach: (1) a PointNet-based~\cite{Qi2017CVPR} model and (2) an extension of this model using photon sampling which improves the quality of shadows and specular reflections.
We demonstrate both theoretically as well as empirically that our model can be trained without bias using noisy renderings from a physically based renderer.

%% file: content/related.tex
\section{Related Work}

\boldparagraph{Rendering}
Physically based rendering is a well-studied field \cite{Veach1998, Pharr2016} where much of the research in recent years focuses on optimizing different parts of the rendering pipeline \cite{Muller2019TOG,Vicini2019TOG,Ren2013TOG, Rainer19CGF} or denoising of fast noisy renderings \cite{Buades2005CVPR, Michael2017HPG, Chaitanya2017SIGGRAPH}.
Moreover, there is a trend of making rendering algorithms differentiable in order to estimate scene properties \cite{Azinovic2019CVPR, Loper2014ECCV, Gkioulekas2013TOG, Gkioulekas2016ECCV} or to use them for training deep neural networks \cite{Valentin2019, NimierDavid2019SIGGRAPHASIA, Li2018TOG, Kato2018CVPR}.
While recent approaches strive to achieve real-time photorealistic rendering \cite{Schied2018PACMCGIT, Schied2017HPG}, they often require additional assumptions such as temporal smoothness and are inherently limited by temporal accumulation of information in screen space.
In this paper, we investigate the suitability of neural networks for learning light transport end-to-end, with the goal of rendering photorealistic images of dynamic scenes in real time.

\boldparagraph{Generative Models}
Recently, deep generative models such as \acp{vae} \cite{Hinton2016Science, Chan2005CPAM, Kingma2014ICLR, Huang2018NIPS} or (conditional) \acp{gan} \cite{Mescheder2018ICML, Karras2018ICLR, Karras2019CVPR, Brock2019ICLR,Isola2017CVPR} have demonstrated that neural networks are capable of generating photorealistic synthetic imagery.
A major limitation of these approaches is that their latent representation is typically rather abstract, making it hard to synthesize consistent images across different perspectives or to manipulate the 3D scene content.
In contrast to the aforementioned approaches,  
our model achieves consistency across viewpoints and scene configurations by exploiting rich 3D representations to approximate light transport.

\boldparagraph{Novel View Synthesis}
Alhaija \etal \cite{Alhaija2018ACCV} and Nalbach \etal \cite{Nalbach2017CGF} describe methods for generating renderings from multiple image buffers such as depth and materials.
While this approach allows for rendering realistic images, a major limitation is that it operates in image space, making it hard to model global illumination correctly.
There also exist several approaches  for novel view synthesis \cite{Park2017CVPR,Tinghui16ECCV,Dosovitskiy2014CVPR, Tatarchenko2016ECCV, Eslami2018Science, Kulkarni2015NIPS, Chen2016NIPS,Worall2017ICCV, Xinchen2016NIPS, Rhodin2018ECCV, Chen2017ICCV, Wang2018CVPRa,Eslami2018Science} that make use of a latent scene representation.
However, since these methods lack a geometric scene representation, it is hard to gain precise control over their output.
In contrast, we learn to render images in a differentiable manner from a \textit{holistic} scene representation.

\boldparagraph{Scene Representations}
Aliev \etal~\cite{Aliev2019ARXIV} propose a neural approach for rendering novel views from point cloud representations \cite{GarciaGarcia2016IJCNN, Qi2017NIPS, Zhang2019ICCV, Hua2018CVPR}.
Meshry \etal~\cite{Meshry2019CVPR} model scenes using point clouds for re-rendering from novel views and under varying appearance.
Hermosilla \etal~\cite{Hermosilla2019CGF} use point clouds for learning abstract features on a scene's surface that can be used for rendering various illumination effects using direct illumination.
While this approach captures complex illumination effects such as subsurface scattering, it is limited to diffuse, homogeneous materials and single objects.
Recently, several alternative scene representations have been considered \cite{Sitzmann2019NIPS,Thies2019TOG, Mescheder2019CVPR, Oechsle2019ICCV, Niemeyer2020CVPR}.
Sitzmann \etal~\cite{Sitzmann2019CVPR} propose \emph{DeepVoxels}, where object-centric static scenes are encoded in a voxel grid of learned features.
While this method works well for rendering a sequence of coherent images, it is limited to compact static scenes due to the high memory requirements of voxel-based representations.
Thies \etal \cite{Thies2019TOG} propose a neural texture representation that can be used for novel view synthesis of objects. %
In contrast to these approaches our model aims to represent multiple scenes as well as scene dynamics by explicitly reasoning about light transport.

%% file: content/method.tex
\def\argmin#1#2{\underset{#1}{\text{arg min }} #2}
\def\br#1{\left(#1\right)}
\def\estimate#1{\hat{#1}}

\def\Set#1{\left\{#1\right\}}
\def\model{\varphi}
\def\norm#1{\left\lVert#1\right\rVert}
\def\expectsymb{\mathbb{E}}
\def\expectsub#1#2{\expectsymb_{#1}\left[#2\right]}
\def\round#1{\left[#1\right]}
\def\abs#1{\left|#1\right|}
\def\imgnoisy{\mathbf{\hat I}}

\begin{figure*}[t!]
	\includegraphics[width=\textwidth]{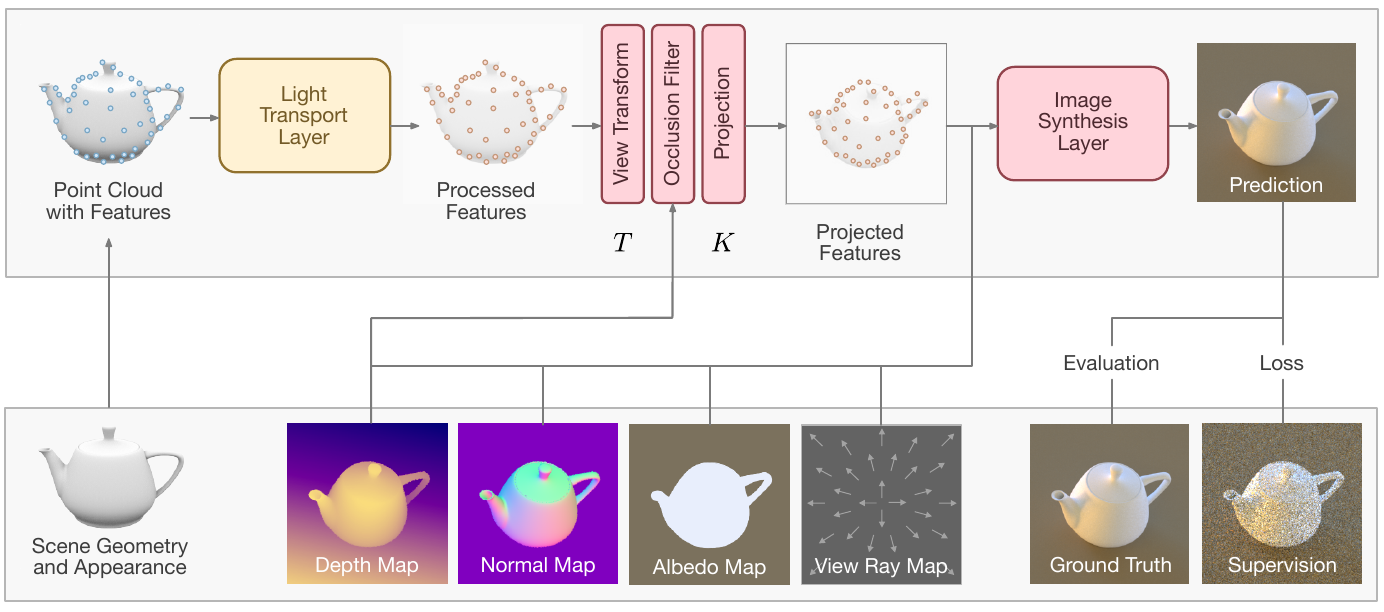}
	\caption{
	\textbf{Model Overview.} 
	Our model samples a point cloud uniformly from the input 3D mesh and associates each point with additional properties (albedo, light spectrum).
	These features are processed using a \textit{Light Transport Layer} which learns to approximate the light transport in the scene.
	The resulting features are projected into the 2D image domain and occluded points
	are removed.
	The final image is synthesized using an \textit{Image Synthesis Layer} that takes the projected features as well as additional image space information as input.
		\vspace{-1em}
	}
	\label{fig:architecture-overview}
\end{figure*}	

\section{Method} \label{sec:method}
	
Our goal is to train a deep neural network $\varphi_\theta$ to render a photorealistic scene specified in terms of a 3D model in real time. We first discuss our scene representation.
Next, we describe our neural rendering architecture, which is able to learn complex illumination effects by exploiting both 3D and 2D information.
Finally, we describe how we train our model using noisy renderings for supervision and show that under moderate assumptions our gradient estimates are unbiased.
An overview of our approach is given in \figref{fig:architecture-overview}.

\boldparagraph{Scene Representation}
How should a 3D scene be represented for efficient and photorealistic rendering?
Traditionally, 3D geometry is often represented in the form of textured 3D meshes.
However, while meshes and texture atlases are compact and encode useful geometric properties, they are inconvenient for neural networks due to their irregular structure.
In contrast, voxel-based representations can be processed conveniently using 3D convolutions, yet they are limited
by their cubic memory requirements.
In this work, we therefore opt for a hybrid 2D-3D representation consisting of both image-space buffers such as albedo, normal and depth maps as well as 3D information.
We represent 3D information in form of an unstructured point cloud sampled from the scene's surface with learned feature embeddings enriched by additional light and material properties.

\boldparagraph{Architecture}
Our neural rendering model comprises three main parts as illustrated in \figref{fig:architecture-overview}: a Light Transport Layer, a 3D-to-2D projection step and an Image Synthesis Layer.
The \textit{Light Transport Layer} models global illumination effects that cannot be modeled in image space:
consider for example a movable lamp which is present in the scene but not visible from the current point of view.
The position, color and intensity of the lamp heavily impact the overall illumination of the scene, but an image-based method, by definition, will fail to reason about these effects.
We therefore propose to reason in the 3D domain.

Our Light Transport Layer takes a set of $N_{\text{surf}}$ randomly sampled 3D surface points $\Set{\bp_j}$ and associated attributes for each point $\Set{\ba_j}$ as input.
These attributes comprise the user-defined material type, the surface albedo, and the light intensity emitted by the point if the point is located on a light-emitting surface.
Our goal is to define an architecture that is able to model or approximate light transport in a scene sufficiently well such that illumination effects like reflections and shadows are predicted correctly.

Towards this goal, we first predict a feature embedding $\mathbf{f}_j$ for each point $\bp_j$ using a PointNet-based architecture \cite{Qi2017CVPR}.
While we found that such a global representation is able to reason about global illumination to some extent, we additionally propose a more explicit model for light transport to model illumination effects more accurately.
Inspired by \emph{photon mapping} \cite{photon-mapping} we sample additional $N_{\text{phot}}$ photon points $\Set{\bq_k}$ from all light sources in the scene.
Photons are randomly cast into the scene and their first intersection with the scene geometry $\Set{\bq_k'}$ are computed. 
For each photon intersection $\bq_k'$ we process the position, color and direction of the initial photon point $\bq_k$ with a fully connected neural network, %
resulting in a feature vector $\mathbf{f}_k$ at $\bq_k'$.
The photon network %
thus encodes information about the light color, intensity and direction which is necessary for photorealistic shading.

Next, we remove occluded points in the 3D scene using the depth map $\bD$ and project the remaining point features $\mathbf{f}_j$ and photon features $\mathbf{f}_k$ onto the image plane using perspective projection $\phi(j) = \round{\bK \bT \bp_j}$ where $\bp_j$ denotes the point location, $\bK$ is the camera matrix and $\bT$ the rigid world-to-view transformation matrix. The resulting 2D feature map $\bF$
is obtained by averaging all points projecting onto the same pixel. Formally, we obtain $\bF_\bu$ at pixel $\bu$ as
\begin{equation} \label{eq:scatter-features}
	\bF_\bu = \begin{cases}
		\frac{1}{\abs{\phi^{-1}(\bu)}}\sum_{j \in \phi^{-1}(\bu)} \bff_j & \text{if } \phi^{-1}(\bu) \neq \emptyset \\
		0, & \text{otherwise}
	\end{cases}
\end{equation}
where $\phi^{-1}(\bu) := \Set{j: \phi(j) = \bu}$ denotes the inverse projection.
We concatenate the resulting feature map $\bF$ with additional image-space information $\bA$, \ie, a depth map, a normal map and an albedo map, which we obtain using OpenGL shaders.
Note that this additional image-space information can be computed cheaply and complements the global scene representation $\bF$ with high-frequency albedo, normal and depth information.
Additionally, we create a view ray map that encodes for each pixel a normalized vector pointing from the camera center to the pixel center in world coordinates.
This information is necessary for learning specular reflection and refraction effects.
The final image synthesis is performed using the \textit{Image Synthesis Layer} which we implement using a conventional U-Net architecture~\cite{Ronneberger2015MICCAI}.

\boldparagraph{Training}
We train our model using a dataset $\cD=\{(\bX_i, \imgnoisy_i)\}$ which comprises pairs of 3D scene representations $\bX_i$ and noisy renderings $\imgnoisy_i$ that are obtained from a physically-based renderer which we run for few iterations.
The input $\bX_i$ consists of a scene represented by a point cloud $\bP_i$, a view represented by a world-to-view transform $\bT_i$ and additional image-space information $\bA_i$.
Our objective is to find a parameter vector $\theta^*$ which minimizes the mean squared error (\ac{mse}) \wrt the model parameters $\theta$:
\begin{equation}
	\theta^* = \argmin{\theta}{  \sum_{i=1}^N  \| \imgnoisy_i - \varphi_\theta(\bT_i, \bP_i, \bA_i) \|^2 }
\end{equation}
Since obtaining clean renderings is very time-consuming, we propose the use of noisy renderings from a physically-based renderer.
A similar technique has recently been used to learn image denoising \cite{Lehtinen2018ICML,Krull2019CVPR}.
Our key insight is that we can exploit the unbiasedness of rendering algorithms like bidirectional path tracing \cite{Lafortune1993} to obtain unbiased gradient estimates.
\begin{lemma}
	Let $\bX$ be an input representation of a scene, $\varphi_\theta$ our rendering network and $\imgnoisy$ a noisy rendering of $\bX$ following a distribution $p(\imgnoisy|\bX)$ which depends on the chosen sampling-based rendering algorithm.
	Assume that the true (noise-free) rendering is given by $\bI(\bX)$. Further assume that the rendering algorithm is unbiased, \ie, $\expectsymb_{\imgnoisy | \bX}[\imgnoisy] = \bI(\bX)$.
	In this case, the following equality holds, \ie, the gradient estimates are unbiased:
	\begin{equation} \label{eq:lemma-unbiased-gradient}
	\expectsub{\imgnoisy | \bX}{ \nabla_{\theta} \|\model_\theta(\bX) - \imgnoisy\|^2 }
	=
	\nabla_{\theta} \|\model_\theta(\bX) - \bI(\bX)\|^2 
	\end{equation}
\end{lemma}
\begin{proof}
	See supplementary material.
\end{proof}

\boldparagraph{Implementation Details}
For the Light Transport Layer, we use a PointNet-based architecture \cite{Qi2017CVPR} with ResNet-blocks \cite{He2016CVPR} of depth two.
For the Image Synthesis Layer we use a UNet \cite{Ronneberger2015MICCAI} with four downsampling and four upsamling blocks.
The network architecture 
used for photon feature creation is a fully-connected ResNet \cite{He2016CVPR} architecture with two residual blocks consisting of two fully-connected layers each.
The input, hidden and output dimension
is the same as for the PointNet architecture.
For training, we use the Adam optimizer \cite{Kingma2015ICLR} with a learning rate of $5\cdot 10^{-4}$ and a batch size of 128 for static scenes and 32 for dynamic scenes (see~\prettyref{sec:experiments}).
The learning rate is decayed exponentially by multiplying it by a factor of $0.99$ after every epoch.
More details are provided in the supplementary material.

%% file: content/experiments.tex
\section{Experiments} \label{sec:experiments}

In our experiments, we investigate the importance of 3D reasoning for learning photorealistic rendering from noisy observations.
We conduct two types of experiments:
In our first set of experiments, we investigate the importance of 3D information and the influence of the different components of the Image Synthesis Layer.
To analyze these properties independently of light transport, we first run our approach on a \textbf{static scene} observed from varying viewpoints.
Our second set of experiments addresses \textbf{dynamic scenes} (moving objects and light sources) using our complete pipeline including the Light Transport Layer.

\boldparagraph{Datasets}
For our experiments on static scenes, we evaluate our approach on a simple static indoor scene containing a table, two light sources and a glass egg \cite{Veach1998}.
Our experiments on dynamic scenes are based on four realistic indoor scenes from \cite{Bittlerli2016}.
We use \emph{Mitsuba} \cite{Jakob2010} for both rendering and point sampling.
Renderings are created using bidirectional path tracing, a modification of path tracing that is unbiased and converges faster \cite{Veach1998}.
For each scene, we create a training set of 100,000 images at a resolution of $256 \times 256$ pixels, varying the camera pose for each training sample.
We sample 10,000 surface points for each scene.
For our experiments on dynamic scenes, we randomly translate or remove objects in addition to varying the camera pose, and sample an additional 10,000 photons and intersections for each scene.
The training data is visualized in the supplementary material.

\boldparagraph{Baselines} 
For our main experiment on dynamic scenes we use three baseline methods: (1) a 2D CNN baseline which predicts images from the image-space input $\bA_i$ alone, 
(2) a denoising approach similar to the model of Lehtinen \etal \cite{Lehtinen2018ICML}, which learns to predict smooth renderings using noisy renderings as input and
(3) a simple feature projection approach similar to Aliev \etal \cite{Aliev2019ARXIV} without Light Transport Layer.
For the denoising approach we trade-off accuracy with run-time by adapting the number of pixels for which we run the bidirectional path tracer.
We report results for $1/1$, $1/4$, $1/16$ and $1/64$ of the total number of image pixels with four samples per pixel, setting all other pixels to black.
For fair comparison, we use the same convolutional architecture for all baselines and our image synthesis layer.

\boldparagraph{Metrics}
For quantitative comparison, we evaluate \ac{mse} and \ac{mssim} \cite{Wang2004TIP} with a window size of $7\times 7$ pixels.
\ac{mse} and \ac{mssim} measure mostly low-level similarity.
To also measure perceptual similarity, we compute the FID \cite{Heusel2017NIPS} and a Feature-L1 distance~\cite{Oechsle2019ICCV} between predicted and ground truth images.
For both the FID and Feature-L1 distance, we use the features of the final average pooling layer of an Inception v3 network~\cite{Szegedy2015CVPR, Szegedy2016CVPR} trained on ImageNet~\cite{Deng2009CVPR}.

\subsection{Ablation Study on Static Scene} \label{sec:ablation-static}

In this section, we conduct experiments on a static scene that does not contain moving objects or light sources.
Our primary goal is to investigate the influence of the different elements of the Image Synthesis Layer as well as the importance of 3D information.

We compare the performance of our model without Light Transport Layer for different input modalities.
\figref{fig:ablation-static} shows the different configurations which are evaluated against each other.
We choose a subset of 6 (out of $2^4=16$) representative configurations to highlight the importance of each input.
While configuration 1, 2 and 3 use only 3D information (but no image space information), configuration 4 and 5 rely solely on image space information.
Finally, configuration 6 combines both 3D and image space information.

\def\veachimagegroup#1{
	\includegraphics[width=.12\textwidth]{compressed/static/teacher/#1.jpg} &
	\includegraphics[width=.12\textwidth]{compressed/static/01/#1.jpg} &
	\includegraphics[width=.12\textwidth]{compressed/static/02/#1.jpg} &
	\includegraphics[width=.12\textwidth]{compressed/static/03/#1.jpg} &
	\includegraphics[width=.12\textwidth]{compressed/static/04/#1.jpg} &
	\includegraphics[width=.12\textwidth]{compressed/static/05/#1.jpg} &
	\includegraphics[width=.12\textwidth]{compressed/static/06/#1.jpg} &
	\includegraphics[width=.12\textwidth]{compressed/static/GT/#1.jpg}
}

\begin{figure*}[t]
	\centering
	\setlength\tabcolsep{1pt} %
	\begin{tabular}{cccccccc}
		\veachimagegroup{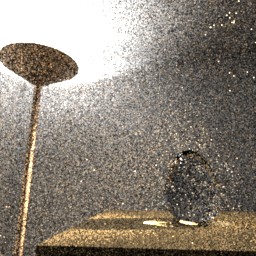} \\
		\footnotesize \makecell{Teaching\\ Input} &
		\footnotesize 1 & \footnotesize 2 & \footnotesize 3 & \footnotesize 4 & \footnotesize 5 & \footnotesize 6 &
		\footnotesize \makecell{Ground\\ Truth}
	\end{tabular}
	\setlength{\tabcolsep}{9pt}
	\vspace{0.3cm}
	\begin{footnotesize}
		\newline\vspace*{1em}
		\begin{tabular}{l|cccc|ccc}
			\toprule
			config & position & point features & normal map & ray direction map &       MSE & MSSIM &    FID \\
			\midrule
			1 &      yes &        no &         no &   no &  0.0106 &  0.81 &  154.6 \\
			2 &       no &       yes &         no &   no &  0.0107 &  0.80 &  158.6 \\
			3 &      yes &       yes &         no &   no &  0.0108 &  0.81 &  149.4 \\
			4 &       no &        no &        yes &   no &  0.0161 &  0.78 &  138.1 \\
			5 &       no &        no &        yes &  yes &  0.0107 &  0.83 &  124.0 \\
			6 &      yes &       yes &        yes &  yes &  \textbf{0.0084} &  \textbf{0.88} &  \textbf{86.1} \\
			\bottomrule
		\end{tabular}
	\end{footnotesize}
	\vspace{0.3cm}
	\caption{\textbf{Ablation Study on Static Scene.}
		Comparing different input configurations for a static scene.
		The metrics are evaluated on a separate held-out validation set comprising 2048 samples.
		All networks were trained for 200,000 iterations with a batch size of 128. Note how the full model (6) predicts images that are significantly less noisy than the teaching input (left). Additional qualitative results are provided in the supplementary material.
		\label{fig:ablation-static}
	}
\end{figure*}

\boldparagraph{Results}
Configurations 1, 2, 3 and 5 show similar performance in terms of \ac{mse}, while configuration 5, which does not receive any projected point cloud information as input, clearly outperforms the other three configurations in terms of \ac{mssim} and \ac{fid}.
However, surface normal information only yields good results if supplemented by viewpoint information, as becomes evident when comparing configurations 4 and 5.
The most important insight is that all inputs in combination (configuration 6), outperform the other configurations for all metrics by a large margin.
This result supports our initial hypothesis that reasoning in both 3D and 2D is crucial for this task.
\figref{fig:ablation-static} (top) shows a qualitative result.
While configurations 1, 2 and 3 achieve reasonable qualitative results, they also contain several artifacts (\eg, the table) which do not occur in configuration 6.
Configurations 4 and 5 do not exploit 3D information, thus severely degrading visual fidelity. This highlights the importance of 3D information for learning-based rendering.
Configuration 6 which uses both 2D image space as well as 3D information yields the best qualitative results.

\def\staticimagegroup#1{
	\includegraphics[width=.155\linewidth]{compressed/04_static_bathroom/staticbath_#1_00.jpg} &
	\includegraphics[width=.155\linewidth]{compressed/04_static_bathroom/staticbath_#1_01.jpg} &
	\includegraphics[width=.155\linewidth]{compressed/04_static_bathroom/staticbath_#1_02.jpg}
}

\def\imagerow#1#2{
	\includegraphics[width=.12\linewidth]{#1/sv/#2} &
	\includegraphics[width=.12\linewidth]{#1/denoise/#2} &
	\includegraphics[width=.12\linewidth]{#1/denoise64/#2} &
	\includegraphics[width=.12\linewidth]{#1/cnn/#2} &
	\includegraphics[width=.12\linewidth]{#1/projection/#2} &
	\includegraphics[width=.12\linewidth]{#1/pointnet/#2} &
	\includegraphics[width=.12\linewidth]{#1/photon/#2} &
	\includegraphics[width=.12\linewidth]{#1/gt/#2} \\
}

\begin{figure}[t]
	\centering
	\setlength\tabcolsep{1pt} %
	\renewcommand{\arraystretch}{.7}
	\begin{tabular}{cccccccc}
		\imagerow{compressed/d1_256/it00300000}{00067_zoom.jpg}
		\imagerow{compressed/d1_256/it00300000}{00078_zoom.jpg}
		\footnotesize \makecell{Teaching\\ Input} &
		\footnotesize \makecell{Denoising\\  (1/1)} & 
		\footnotesize \makecell{Denoising\\  (1/64)}  & 
		\footnotesize \makecell{CNN only} & 
		\footnotesize \makecell{Feature \\ Projection} & 
		\footnotesize \makecell{Ours w/o\\ Photons} & 
		\footnotesize \makecell{Ours w/\\ Photons} &
		\footnotesize \makecell{Ground\\ Truth} \\
		\vspace{0.01cm}
	\end{tabular}
	\setlength{\tabcolsep}{13pt}
	\renewcommand{\arraystretch}{1}
	\begin{footnotesize}
		\begin{tabular*}{\linewidth}{lccccc}
			\toprule
			Architecture &   time / frame &    MSE ($\downarrow$) &  MSSIM ($\uparrow$) & FID ($\downarrow$) &  Feature L1 ($\downarrow$) \\
			\midrule
			Denoising (1/1) &  1.5059s &  0.0005 &  0.880 &   26.4 &                 0.163 \\
			\midrule
			Denoising (1/64) &  0.0283s &  0.0029 &  0.781 &   94.0 &                 0.281 \\
			CNN only &  0.0191s &  0.0043 &  0.835 &   36.1 &                 0.195 \\
			Feature Projection &  0.0210s &  0.0037 &  0.841 &   32.5 &                 0.185 \\
			Ours (w/o Photons) &  0.0243s &  0.0044 &  0.841 &   31.4 &                 0.184 \\
			Ours (w/ Photons) &  0.0459s &  \textbf{0.0028} &  \textbf{0.849} & \textbf{30.6} &  \textbf{0.182} \\
			\bottomrule
		\end{tabular*}
	\end{footnotesize}
	\caption{\textbf{Dynamic Objects and Fixed Lights.}
		Results on dynamic scenes where objects are modified but light sources kept fixed. We  show the non-real-time denoising baseline ``Denoising (1/1)'' for reference. Additional results are provided in the supplementary material.
	}
	\label{fig:dynamic-scenes}
\end{figure}

\begin{figure}[th!]
	\centering
	\includegraphics[width=0.4\linewidth]{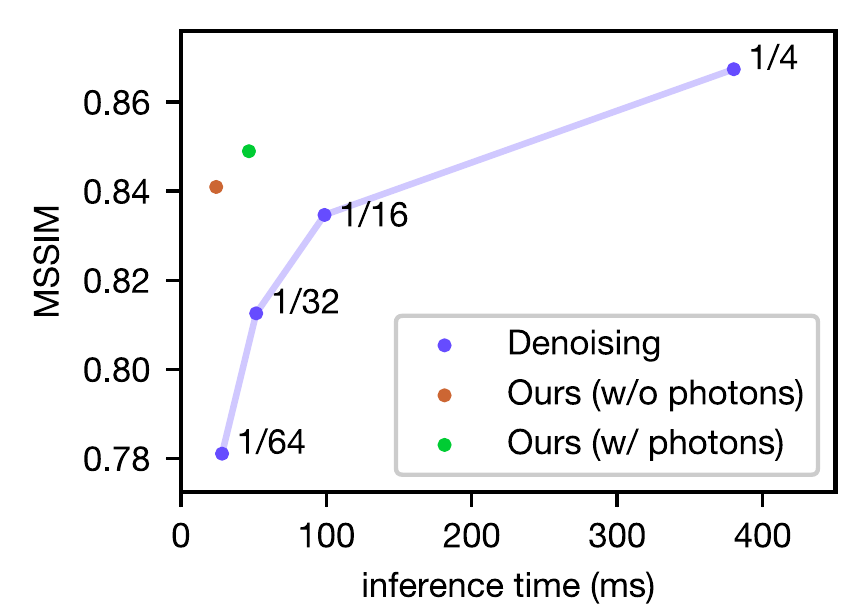}\hspace{0.7cm}%
	\includegraphics[width=0.4\linewidth]{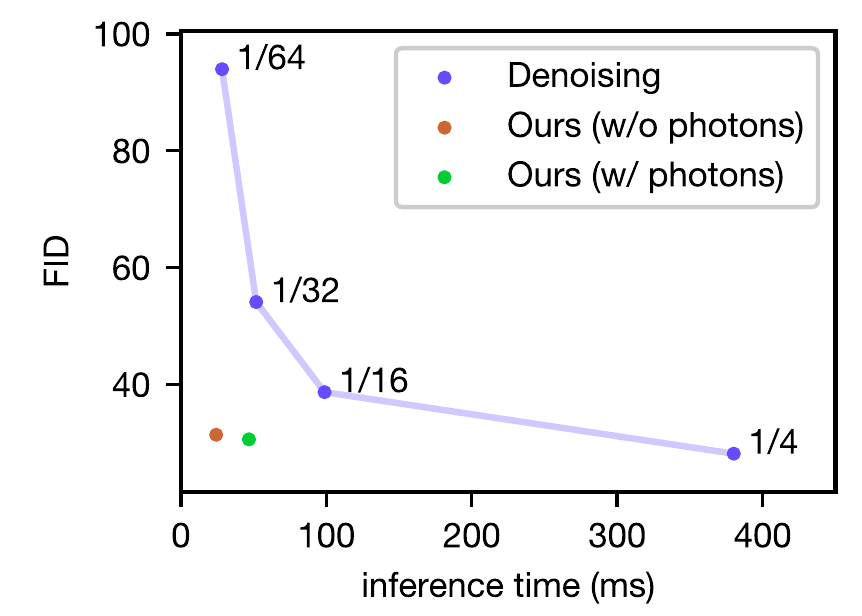}\\%
	\caption{\textbf{Dynamic Objects and Fixed Lights.}
		Quantitative comparison of our approach to the denoising baseline, varying the sample density.
		We plot reconstruction accuracy in terms of MSSIM and FID over inference time.
		Numbers refer to the ratio of dropped pixels.
	}
	\label{fig:denoising-graph}
\end{figure}

\subsection{Results on Dynamic Scenes} \label{sec:dynamic-scenes}

To investigate the utility of 3D reasoning, 
we now turn our attention to dynamic scenes where objects (and light sources) are modified.

\subsubsection{Dynamic Objects and Fixed Lights} \label{sec:dynamic-objects-fixed-lights}

We first train our network on a set of four scenes where objects are randomly removed or translated in the scene, but keep all light sources fixed.

\boldparagraph{Results}
\figref{fig:dynamic-scenes} shows qualitative and quantitative results for our approach and the baselines.
We clearly see that our full model which uses both the Light Transport and the Image Synthesis Layers outperforms the other real-time approaches (lower section of the table), both qualitatively and quantitatively in terms of MSE, MSSIM, FID and Feature-L1 distance.
While the non-real-time denoising approach ``Denoising (1/1)'' achieves the best results, the real-time denoising approach that uses much fewer samples performs the worst.
We further analyze this behavior by plotting the MSSIM as a function of rendering time in \figref{fig:denoising-graph}.
While denoising approaches are able to achieve compelling results, the proposed neural rendering approach provides a better accuracy/runtime trade-off while being fully differentiable.

As evident from \figref{fig:dynamic-scenes}, our simple feature projection baseline performs only slightly weaker than our variant without photon mapping.
We attribute this to the fact that most of the light field in the scene can be encoded in local features and only dynamic parts like sharp shadows have to be learned.
This highlights the capability of neural rendering approaches to learn useful heuristics from the training data.
However, we also observe that our full architecture with photon mapping (which reasons more explicitly about light transport) achieves by far the best quantitative results.

\def\imagerowdiff#1#2{
	\includegraphics[width=.118\linewidth]{#1/projection/#2} &
	\includegraphics[width=.118\linewidth]{#1/projection_diff/#2} &
	\includegraphics[width=.118\linewidth]{#1/pointnet/#2} &
	\includegraphics[width=.118\linewidth]{#1/pointnet_diff/#2} &
	\includegraphics[width=.118\linewidth]{#1/photon/#2} &
	\includegraphics[width=.118\linewidth]{#1/photon_diff/#2} &
	\includegraphics[width=.118\linewidth]{#1/gt/#2} \\
}

\begin{figure*}[t]
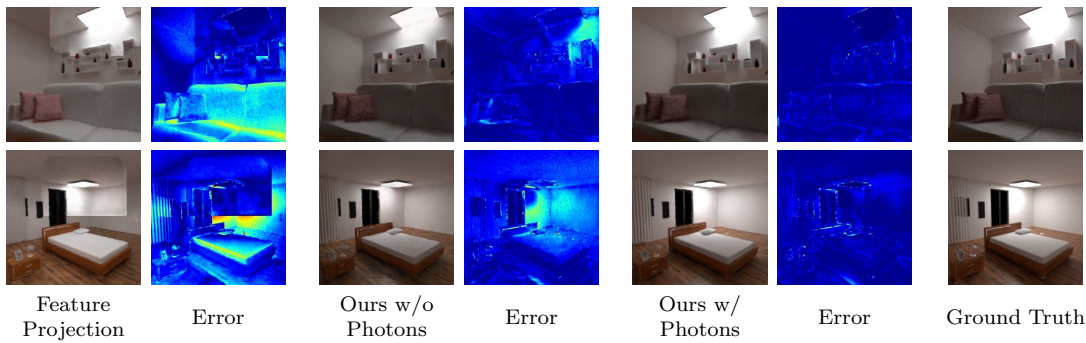

	\centering
	\small
	\setlength\tabcolsep{2pt} %
	\begin{tabular}{cc>{\hspace{9pt}}cc>{\hspace{9pt}}cc>{\hspace{9pt}}c}
		\imagerowdiff{compressed/d2_256/it00300000}{00005.jpg}
		\imagerowdiff{compressed/d2_256/it00300000}{00111.jpg}
		\footnotesize \makecell{Feature \\ Projection} & 
		\footnotesize Error &
		\footnotesize \makecell{Ours w/o\\ Photons} & 
		\footnotesize Error &
		\footnotesize \makecell{Ours w/\\ Photons} &
		\footnotesize Error &
		\footnotesize Ground Truth
	\end{tabular}
	\caption{\textbf{Dynamic Objects and Dynamic Lights.}
		We show the output of the feature projection baseline and our network's predictions with and without photons alongside the corresponding error maps for moving light sources. See supplementary for more results.
	}
	\label{fig:photon-importance}
\end{figure*}

\subsubsection{Dynamic Objects and Dynamic Lights} \label{sec:dynamic-objects-dynamic-lights}

In the previous experiment, both the feature projection and our approach without photons were able to handle shadows and other illumination effects well.
The reason for this is that the light sources were assumed static, making it possible to encode viewpoint-dependent light properties into the point features.
However, by design the feature projection baseline is unable to acquire an understanding of illumination effects in the presence of movable light sources that are not present in the current view.  
To see this effect, we augment the dataset from the previous experiment by turning all static light sources off and replacing them with a rectangular area light at the ceiling, which we move randomly.

\boldparagraph{Results}
Results from our method with photons, our approach without photons and the feature projection baseline are shown in \figref{fig:photon-importance}.
We observe that the feature projection baseline produces considerable artifacts while our approach with photons leads to much sharper shadows and more consistent global illumination.
This is also evident from the error maps in~\figref{fig:photon-importance}.
We provide a full quantitative evaluation in the supplementary material.

%% file: content/conclusion.tex
\section{Conclusion}

In this work, we have systematically investigated the importance of 3D \vs 2D reasoning for learning based photorealistic rendering.
Our experiments demonstrate that neural rendering benefits from joint 3D-2D reasoning, also confirming our hypothesis that reasoning in 3D is helpful in the presence of moving objects and light sources.
In contrast to denoising methods which rely on outputs from a sampling-based renderer, the presented approach is fully differentiable and can be used for training deep neural networks end-to-end.

\clearpage

%% file: content/supplementary/abstract.tex
\section*{Supplementary Material}

This supplementary document provides additional information on our approach and more experimental results.
First, we provide detailed information on the Light Transport and Image Synthesis Layers in \secref{sec:architectures}.
We then describe the data generation pipeline in more detail in \secref{sec:data}.
Afterwards, we provide more information on the training procedure in \secref{sec:training}, including a proof showing
   	that we can train our models using noisy, unbiased renderings as supervision signal.
Finally, we provide additional qualitative and quantitative results in \secref{sec:results}.

%% file: content/supplementary/architectures.tex
\section{Architectures} \label{sec:architectures}

\subsection{Light Transport Layer}
The core of the Light Transport Layer is a PointNet-based architecture \cite{Qi2017CVPR} with fully-connected ResNet blocks \cite{He2016CVPR}, which is illustrated in \figref{fig:pointnet}.
While the PointNet architecture can have arbitrary depth (number of ResNet blocks), we use a depth of two for all the experiments in the paper.

Since we train our model on dynamic scenes with a variable number number of visible objects, the input point clouds have different sizes for different training samples.
In theory this is not a big problem, as PointNets can handle arbitrary point cloud sizes.
However, since we are using mini batches for training, having the same number of points for each training sample is desirable.
Therefore, our model always operates on the maximum point cloud size, and invisible objects are masked in the architecture using per-point visibility flags.

\begin{figure}[p]
\includegraphics[width=\textwidth]{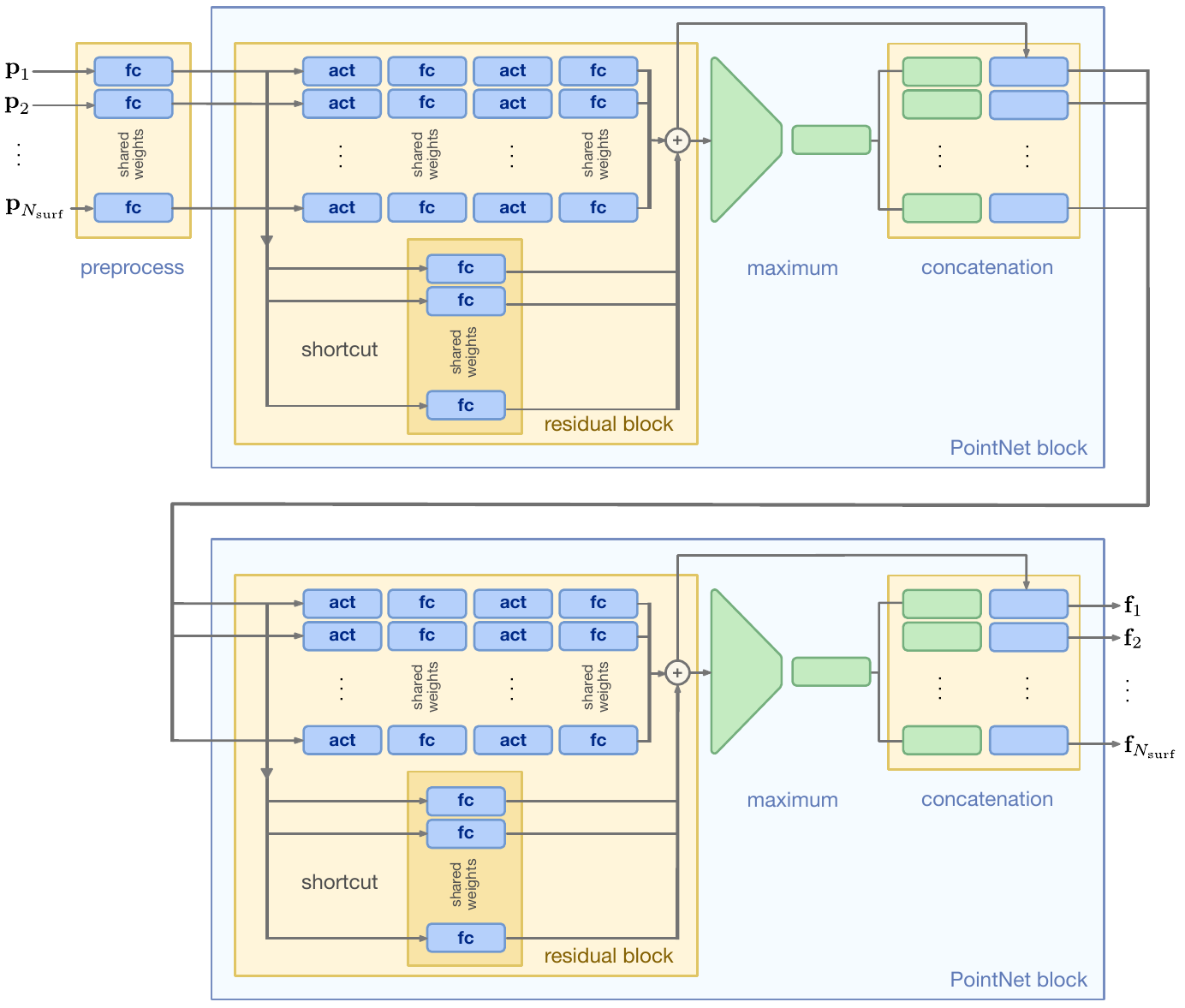}
\caption{
	\textbf{Light Transport Layer.}
	In the first stage of the Light Transport Layer all points and supplementary point information $\bp_i$ are processed in a preprocessing layer, whose purpose is to align its output feature dimension with the input feature dimension of the PointNet block.
	The point features are then processed in two consecutive PointNet blocks.
	A \mbox{PointNet} block comprises a residual block, where local features are computed for each point.
	The fully connected layers (hidden, output and shortcut layers) consist of 32 output neurons each, where the weights within one layer are shared between the input points.
	The input dimension to the first fully connected layer in a residual block is aligned with the output dimension of a PointNet block (64).
	Therefore, a fully connected shortcut layer is required for matching the feature dimensions at the end of a residual block.
	Following the residual block within a PointNet block, point features are concatenated with a global feature, which is computed as the maximum feature vector of all local features.
	The output features of the second PointNet block are denoted by $\mathbf{f}_i$.
	We denote fully connected layers by \textbf{fc} and ReLU activation functions by \textbf{act}.
	\label{fig:pointnet}
}
\end{figure}

\subsection{3D-to-2D Projection Step}
In the 3D-to-2D projection step, the 3D point features are projected to image space, where the point locations are discretized.
Points that are occluded by the scene's geometry are masked out, which is determined by performing an occlusion check using a rendered depth map.
To make sure that we do not accidentally remove points on the scene's surface, we use a tolerance of $\varepsilon=10^{-3}$ in the occlusion check.
If multiple features are projected to the same pixel, we compute the mean feature vector for all points projecting to that pixel.
If a pixel has no points projecting to it, its feature vector is defined as zero.

\subsection{Image Synthesis Layer}

The input to the Image Synthesis Layer are the projected features from the projection step and additional information in image space, which can be computed cheaply using OpenGL shaders
These image space buffers contain information about the geometry and material information observed from the current view.
They include depth map, albedo (diffuse reflectance), normal map in world coordinates as well as a view ray map, which contains for each pixel the ray direction in world coordinates going from the camera center through the respective pixel center.
The intention behind using these image space layers is to leverage the image formation process in multiple ways.
The normal and view direction information can be used by the network to infer shading in image space.
The albedo layer supports texture synthesis where point projections are sparse.
In addition, by providing this information in image space, the light transport layer can solely focus on the task of modeling the illumination in the scene.
However, the image space layers do not contain useful information for reasoning about light transport in the scene.
A detailed visualization and description of the Image Synthesis Layer is provided in \figref{fig:unet}.

\begin{figure}[t!]
\includegraphics[width=\textwidth]{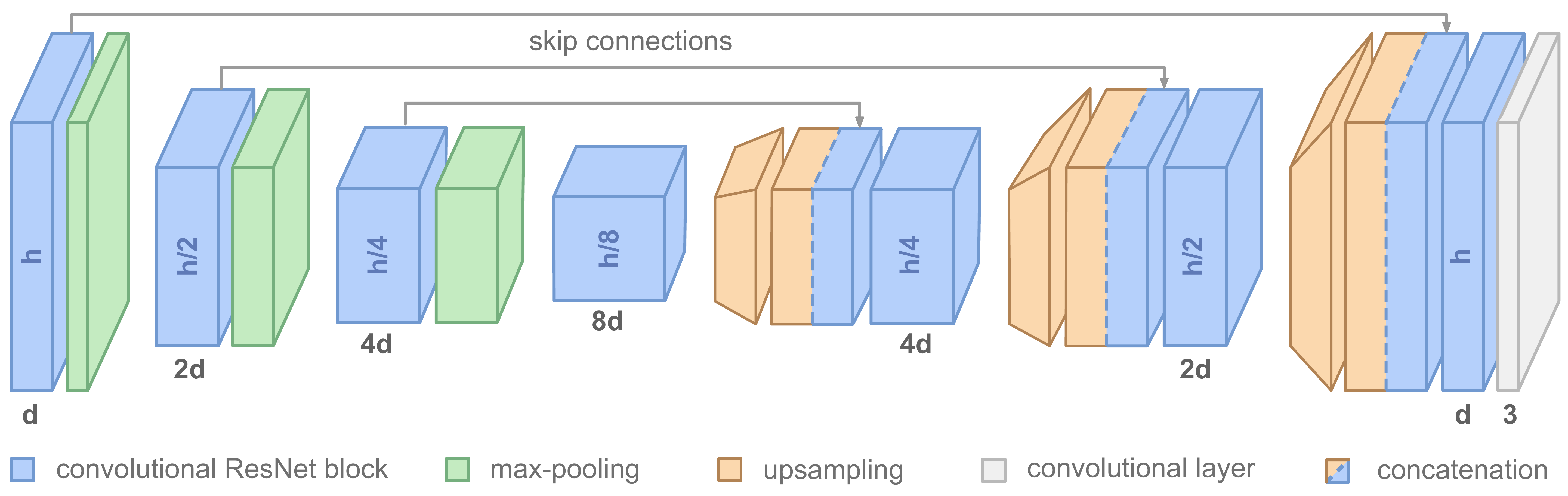}
\caption{
	\textbf{Image Synthesis Layer.}
	For final image synthesis we use a UNet \cite{Ronneberger2015MICCAI} architecture where the resolution is reduced in three steps and expanded again.
	To this end, each level comprises a convolutional ResNet \cite{He2016CVPR} block consisting of two $3\times3$ convolutional layers and ReLu activations.
	We use the same feature dimension for the input, hidden and output layers in a convolutional ResNet block.
	The convolutional ResNet blocks are followed by a downsampling step, which is implemented using max-pooling layers.
	The feature dimension of the convolutional layers depend on the level, starting at a dimension of $d=64$, which is then doubled after each downsampling step.
	The features of the lowest level are then upsampled again using bilinear interpolation, concatenated with the convolutional ResNet block output from the respective downsampling layer through a skip connection and processed in another convolutional ResNet block.
	After the last upsampling layer an additional convolutional layer is used to render an image with three channels.
	The numbers below the layers correspond to the number of feature maps in each layer.
	The numbers inside the layers correspond to the layer's resolution, starting at a square resolution of $h\times h$.
	We use $h=128$ in our static scene ablation study and $h=256$ for the other experiments.
	\label{fig:unet}
}
\end{figure}

%% file: content/supplementary/data.tex
\section{Datasets} \label{sec:data}

\subsection{Data Generation and Sampling}

The datasets used in our experiments comprise a single scene for each static scene dataset, and four scenes for dynamic experiments \cite{Bittlerli2016}.
Since the data generation procedure for static scenes is a simplification of the dynamic case, we only describe the dynamic case in this section.
Since we use learnable feature descriptors in our model, we must ensure that there are point correspondences between different training samples of the same scene.
To this end, we sample an initial, static point cloud for each scene.
This point cloud is then modified according to the scene modifications in the training sample.
If an object is removed from the scene, the points are removed from the initial point cloud.
If an object is translated, the points sampled from its surface are translated accordingly.
For each scene in the dataset, we first sample a static point cloud, which is then modified for each sample in the dataset.
A positive side effect of this is that we only have to store scene modification information for each sample, saving memory.

\subsection{View Sampling}
For each scene, we would like to cover the space of possible viewing locations and directions as accuractely as possible.
At the same time we want to have a high number of views where a lot of scene details are visible to have an effective supervision signal for training.
We observe that most of the objects in a scene are arranged along the walls or the floor.
Therefore, we sample a viewing location uniformly from a bounding box that is slightly smaller than the scene's bounding box.
Note that this means that a few of the sampled locations might lie inside an object.
However, we found that these ``outliers'' do not pose a problem to our method in practice as long as we observe a sufficiently large number of views outside of objects during training.
Next, we sample a viewing direction by sampling a look-at location uniformly from a bounding box that is half the size of the location bounding box.
As a result, the distance between the camera and scene objects is far enough to render views with rich image content.

\subsection{Point Cloud Sampling}
We define a scene by a set of shapes $\mathcal{S}$, where each shape $S_i \in \mathcal{S}$ is itself a set of triangles.
Each shape is assigned a sampling importance $w(S_i)$ corresponding    to its surface area, which is the sum of triangle areas for that shape.
Given the sampling importances and a point cloud of size $N$, we first sample $N$ shapes according to a distribution where the probability of sampling a shape is proportional to its sampling importance.
This can be achieved by using discrete inverse transform sampling, where a discrete cumulative distribution is calculated for the sequence of shapes $(S_1, \hdots, S_n)$:
\begin{equation}
	\text{cdf}(i) = \frac{\sum_{j=1}^i w(S_j)}{\sum_{j=1}^n w(S_j)}
\end{equation}
Using a uniform sample $s \sim \mathcal{U}(0, 1)$, a shape index $i$ can be sampled according to
\begin{equation}
	i = \underset{k}{\text{arg max}} \left\{ k: \text{cdf}(k) < u \right\},
\end{equation}
which can be implemented efficiently using bisection.

For each shape sampled from the distribution, our goal is to obtain a point sampled uniformly from the shape's surface.
Since we work with a mesh scene representation, all the shapes are represented by a set of triangles.
Therefore, for each point we first sample the triangle with the same technique we used for shape sampling, using the triangle area as sampling importance.
Then, we sample a point location uniformly from the triangle.
This way, uniformly distributed samples from the shapes' surfaces can be obtained.

\subsection{Scene Modification Sampling}
To train our model on all possible scene configurations (each object or light source could be located anywhere or not be present in the scene at all), we must cover this distribution well in the dataset.
To this end, we manually define for each dynamic object an axis-aligned bounding box from which we sample a position for each training sample.
The bounding boxes can also be limited to one or two dimensions, e.g. if an object can only be translated along a wall.
Although we do not always get realistic object arrangements using this sampling strategy, this is not a limitation, as it makes our model more general (\ie our model is trained for both realistic and non-realistic object arrangements).
In addition to object translations, we randomly remove objects with a probability of 0.2 per object from the scene.

\begin{figure}[t!]
	\includegraphics[width=\textwidth]{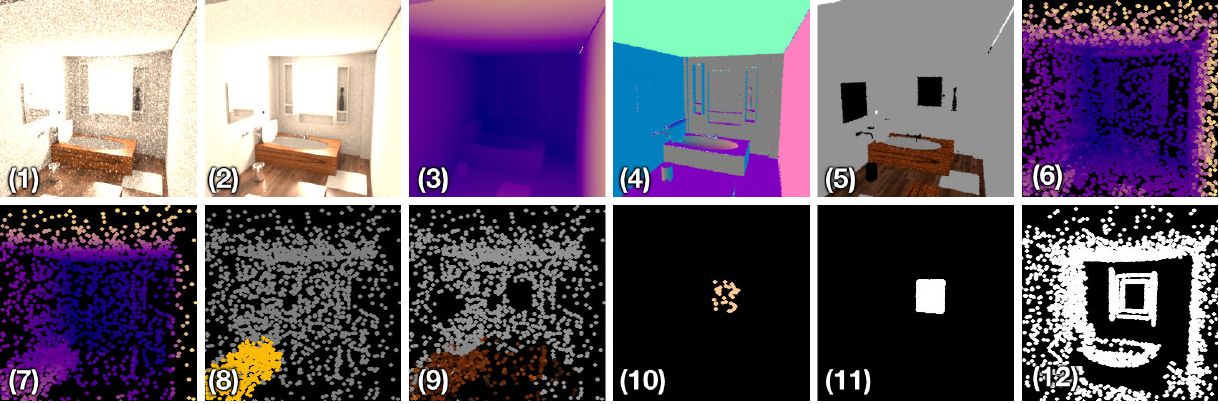}
	\caption{\textbf{Data Sample Components.}
		(1) noisy supervision, (2) ground truth rendering, (3) depth map, (4) normal map, (5) albedo map, (6) point cloud, (7) point cloud with occlusion masking, (8) per-point visibility, (9) per-point albedo, (10) per-point emitter spectrum, (11) photon origins, (12) photon intersections with scene.
		\label{fig:data-samples}
	}
\end{figure}

\subsection{Data Sample Components} \label{sec:sample-components}

In \figref{fig:data-samples} we provide a visualization of the different components of a sample in our dataset.
A sample consists of a noisy supervision rendering, which we obtain from a bidirectional path tracer \cite{Jakob2010}, using one sample per pixel (image (1)).
For the evaluation of the test set, we render an additional ground truth rendering using 128 samples per pixel (image (2)).
The images (3)--(5) in \figref{fig:data-samples} are visualizations of the additional image-space layers as described in \secref{sec:method}.
Image (6) visualizes all points in the point cloud for that sample, and image (7) all points that are visible in the current view.
Image (8) visualizes the per-point visibility of objects in the scene with yellow denoting invisible objects.
The per-point visibility is needed as our architecture requires a fixed number of points as input, and is used for masking out invisible objects in the Light Transport Layer.
Images (9) and (10) visualize additional point properties which are fed into the Light Transport Layer, such as the diffuse reflectance for each point as well as an emitter spectrum, which is non-zero for points lying on a light source.
Images (11) and (12) show photon origins and their respective first intersection with the scene.

%% file: content/supplementary/training.tex
\section{Training} \label{sec:training}

\subsection{Hyperparameters}
We train all models using the Adam optimizer \cite{Kingma2015ICLR} with a learning rate of $\lambda=5 \times 10^{-4}$, which we decay by a factor of $0.99$ after each epoch.
These hyperparameters are the result of a hyperparameter optimization using grid search, where we tested different learning rates and decay rates for Adam and RMSprop for 100,000 iterations.
For the static scene experiment in \secref{sec:ablation-static} we use a batch size of 128.
For the dynamic scene experiments in \secref{sec:dynamic-scenes} we use a batch size of 32, as more GPU memory is required for the Light Transport Layer implementation.
All models were created and trained using PyTorch 1.0\footnote{\url{https://pytorch.org}}  \cite{Paszke2017}.

\subsection{Supervised Learning with Noisy Renderings}
\def\params{\theta}
\def\Set#1{\left\{#1\right\}}
\def\round#1{\left[#1\right]}
\def\proj{\phi}
\def\inlinefrac#1#2{#1/#2}
\def\abs#1{\left|#1\right|}
\def\conv{\odot}
\def\defeq{:=}
\def\vectorbf#1{\mathbf{#1}}
\def\pointbf#1{\mathbf{#1}}

\def\diff{\text{d}}
\def\br#1{\left(#1\right)}

\def\expectsymb{\mathbb{E}}
\def\expect#1{\expectsymb\left[#1\right]}
\def\expectalt#1#2{\expectsymb_{#1}#2}
\def\expectsub#1#2{\expectsymb_{#1}\left[#2\right]}
\def\expectsubsq#1#2{\expectsymb_{#1}^2\left[#2\right]}
\def\var{\text{Var}}
\def\cov#1{\text{cov}\br{#1}}
\def\trace{\text{Tr}}

\def\norm#1{\left\lVert#1\right\rVert}
\def\scalarprod#1{\left\langle#1\right\rangle}

\def\Set#1{\left\{#1\right\}}
\def\model{\varphi}
\def\modelx{\model_\theta(\bX)}

\def\norm#1{\left\lVert#1\right\rVert}
\def\expectsymb{\mathbb{E}}
\def\expectsub#1#2{\expectsymb_{#1}\left[#2\right]}
\def\expi{\expectsub{\imgnoisy|\bI}}
\def\round#1{\left[#1\right]}
\def\abs#1{\left|#1\right|}
\def\imgnoisy{\mathbf{\hat I}}

Since rendering a large set of photorealistic renderings for training would require a lot of time, we use noisy renderings from a physically based renderer as supervision. 
More specifially, we use the bidirectional path tracing implementation in Mitsuba \cite{Jakob2010}.
Similar techniques have recently been used to learn image denoising \cite{Lehtinen2018ICML,Krull2019CVPR}.
Our key insight is that we can exploit the unbiasedness of state-of-the-art rendering algorithms like bidirectional path tracing \cite{Lafortune1993} to obtain unbiased gradient estimates.

To this end, we describe the input to our network by random variable $\bX$, which comprises
a point cloud $\bP$, a view represented by a world-to-view transform $\bT$ and additional image-space information $\bA$
as described in \secref{sec:method}.
As supervision signal, we render a noisy image $\imgnoisy$ that is an unbiased estimate of the ground truth rendering $\bI(\bX)$.
When we train our network using the \ac{mse} and stochastic gradient descent, our gradients will be unbiased when using these noisy supervision renderings from such an unbiased rendering algorithm.
Formally, this can be expressed as follows:
\begin{lemma}
	Let $\bX$ be an input representation of a scene, $\varphi_\theta$ our rendering network and $\imgnoisy$ a noisy rendering of $\bX$ following a distribution $p(\imgnoisy|\bX)$ which depends on the chosen sampling-based rendering algorithm.
	Assume that the true (noise-free) rendering is given by $\bI(\bX)$. Further assume that the rendering algorithm is unbiased, \ie, $\expectsymb_{\imgnoisy | \bX}[\imgnoisy] = \bI(\bX)$.
	In this case, the following equality holds, \ie, the gradient estimates are unbiased:
	\vspace{-0.3em}
	\begin{equation} \label{eq:lemma-unbiased-gradient}
		\expectsub{\imgnoisy | \bX}{ \nabla_{\theta} \|\model_\theta(\bX) - \imgnoisy\|^2 }
		=
		\nabla_{\theta} \|\model_\theta(\bX) - \bI(\bX)\|^2 
	\end{equation}
	\vspace{-1em}
\end{lemma}
\begin{proof}
Since the expectation does not depend on the parameters $\theta$, the gradient can be pulled out of the expectation. The left side of \eqnref{eq:lemma-unbiased-gradient} becomes
\begin{align}\label{eq:proof1}
\expectsub{\imgnoisy | \bX}{
	 \nabla_{\theta}  \|\model_\theta(\bX) - \imgnoisy\|^2 }
=   \nabla_{\theta}
\expectsub{\imgnoisy | \bX}{ 
	\|\model_\theta(\bX) - \imgnoisy\|^2 }
\end{align}
By applying the binomial theorem and the property of the estimator $\imgnoisy$ being unbiased, which means that $\expectsub{\imgnoisy | \bX}{\imgnoisy}=\bI(\bX)$,
the expectation term can be further expanded to
\begin{align}
	\expectsub{\imgnoisy | \bX}{  \|\model_\theta(\bX) - \imgnoisy\|^2 }
	& =
	\expectsub{\imgnoisy | \bX}{
		\|\model_\theta(\bX)\|^2  - 2 \langle \model_\theta(\bX),
		\imgnoisy\rangle + \|\imgnoisy\|^2 }
	\\ & = 
		\expectsub{\imgnoisy | \bX}{\|\model_\theta(\bX)\|^2}  
		- 2\, \expectsub{\imgnoisy | \bX}{\langle \model_\theta(\bX), \imgnoisy\rangle}
		 + \expectsub{\imgnoisy | \bX}{\|\imgnoisy\|^2 }
	\\ & =  \|\model_\theta(\bX)\|^2
		- 2\, \langle \model_\theta(\bX), \bI(\bX)\rangle
		+ \expectsub{\imgnoisy | \bX}{\|\imgnoisy\|^2 }
		  \label{eq:proof7}
\end{align}
Taking the gradient with respect to $\theta$ in \eqnref{eq:proof7} allows for removing or adding terms that are constant with respect to $\theta$. Thus, we can replace 
$\expectsub{\imgnoisy | \bX}{\|\imgnoisy\|^2 }$ 
with $\|\bI(\bX)\|^2 $:
\begin{align}
	\nabla_{\theta} \expectsub{\imgnoisy | \bX}{  \|\model_\theta(\bX) - \imgnoisy\|^2 }
	& =
	\nabla_{\theta} \left[\|\model_\theta(\bX)\|^2
		- 2\, \langle \model_\theta(\bX), \bI(\bX)\rangle
		+ \expectsub{\imgnoisy | \bX}{\|\imgnoisy\|^2 }\right]
	\\ & =  \nabla_{\theta} \left[\|\model_\theta(\bX)\|^2
	- 2\, \langle \model_\theta(\bX), \bI(\bX)\rangle
	+ \|\bI(\bX)\|^2 \right]
	\\ & =  \nabla_{\theta} \| \model_\theta(\bX) - \bI(\bX) \|^2
\end{align}
Inserting this into \eqnref{eq:proof1} results in \eqnref{eq:lemma-unbiased-gradient}, concluding the proof.
\end{proof}

%% file: content/supplementary/results.tex
\section{Additional Results} \label{sec:results}

\def\veachimagegroup#1{
	\includegraphics[width=.12\textwidth]{compressed/static/teacher/#1.jpg} &
	\includegraphics[width=.12\textwidth]{compressed/static/01/#1.jpg} &
	\includegraphics[width=.12\textwidth]{compressed/static/02/#1.jpg} &
	\includegraphics[width=.12\textwidth]{compressed/static/03/#1.jpg} &
	\includegraphics[width=.12\textwidth]{compressed/static/04/#1.jpg} &
	\includegraphics[width=.12\textwidth]{compressed/static/05/#1.jpg} &
	\includegraphics[width=.12\textwidth]{compressed/static/06/#1.jpg} &
	\includegraphics[width=.12\textwidth]{compressed/static/GT/#1.jpg}
}

\begin{figure}[t]
	\centering
	\setlength\tabcolsep{1pt} %
	\begin{tabular}{cccccccc}
		\veachimagegroup{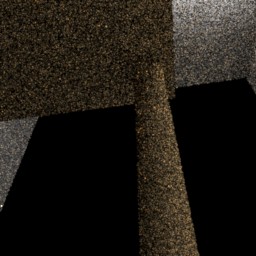} \\[-0.2em]
		\veachimagegroup{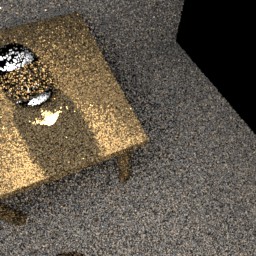} \\[-0.2em]
		\veachimagegroup{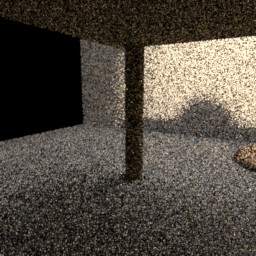} \\[-0.2em]
		\veachimagegroup{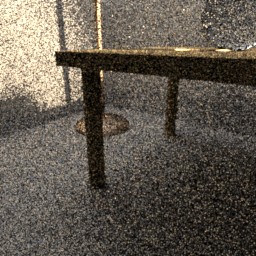} \\[-0.2em]
		\veachimagegroup{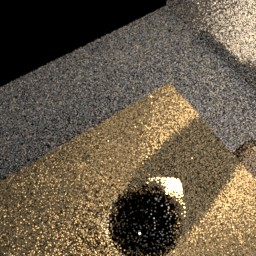} \\[-0.2em]
		\veachimagegroup{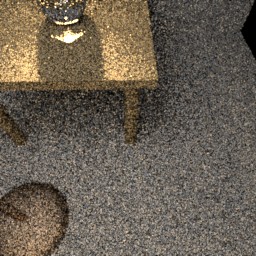} \\
		\footnotesize \makecell{Teaching\\ Input} &
		\footnotesize 1 & \footnotesize 2 & \footnotesize 3 & \footnotesize 4 & \footnotesize 5 & \footnotesize 6 &
		\footnotesize \makecell{Ground\\ Truth}
	\end{tabular}
	\caption{\textbf{Ablation Study on Static Scene.}
		Additional visual results for our static scene ablation study, extending the results in \figref{fig:ablation-static} in \secref{sec:ablation-static}. 
		\label{fig:ablation-static-supp}
	}
\end{figure}

\def\staticimagegroup#1{
	\includegraphics[width=.155\linewidth]{compressed/04_static_bathroom/staticbath_#1_00.jpg} &
	\includegraphics[width=.155\linewidth]{compressed/04_static_bathroom/staticbath_#1_01.jpg} &
	\includegraphics[width=.155\linewidth]{compressed/04_static_bathroom/staticbath_#1_02.jpg}
}
\def\staticimagegroupkitchen#1{
	\includegraphics[width=.155\linewidth]{compressed/04_static_kitchen/statickitchen_#1_00.jpg} &
	\includegraphics[width=.155\linewidth]{compressed/04_static_kitchen/statickitchen_#1_01.jpg} &
	\includegraphics[width=.155\linewidth]{compressed/04_static_kitchen/statickitchen_#1_02.jpg}
}

\begin{figure}[p]
	\centering
	\setlength\tabcolsep{1pt} %
	\renewcommand{\arraystretch}{.7}
	\begin{tabular}{ccc>{\hspace{7.5pt}}ccc}
		\staticimagegroup{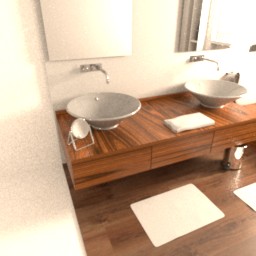} & 
		\staticimagegroup{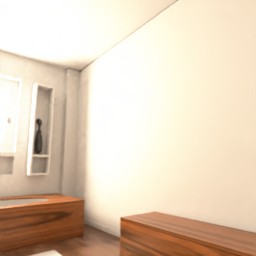} \\
		\staticimagegroup{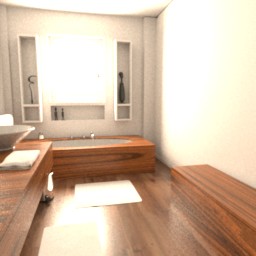} &
		\staticimagegroup{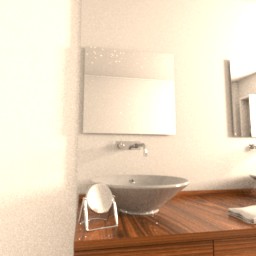} \\
		\staticimagegroup{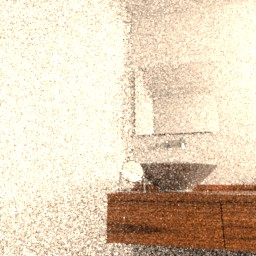} &
		\staticimagegroup{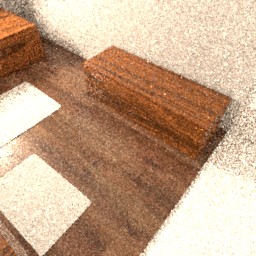} \\
		\staticimagegroup{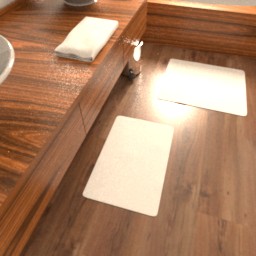} &
		\staticimagegroup{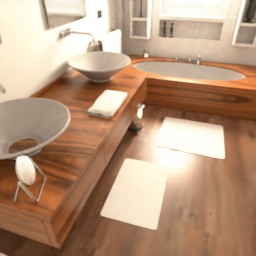} \\
		\footnotesize \makecell{Teaching Input} &
		\footnotesize \makecell{Prediction} &
		\footnotesize \makecell{Ground Truth} &
		\footnotesize \makecell{Teaching Input} &
		\footnotesize \makecell{Prediction} &
		\footnotesize \makecell{Ground Truth} \\
	\end{tabular}
	\caption{
		\textbf{Bathroom Scene.}
		Results of our model on a realistic static bathroom scene.
		\label{fig:results-static-bathroom}
	}
	\vspace{0.8em}
	\begin{tabular}{ccc>{\hspace{7.5pt}}ccc}
		\staticimagegroupkitchen{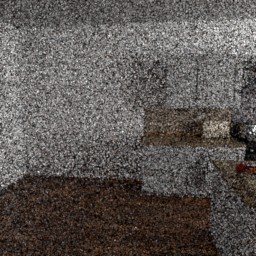} &
		\staticimagegroupkitchen{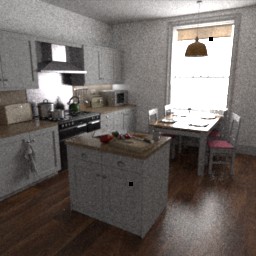} \\
		\staticimagegroupkitchen{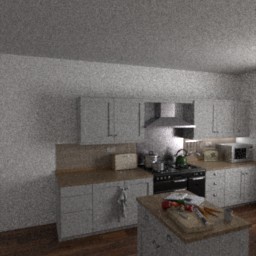} &
		\staticimagegroupkitchen{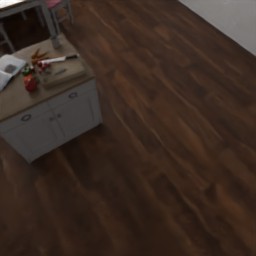} \\
		\staticimagegroupkitchen{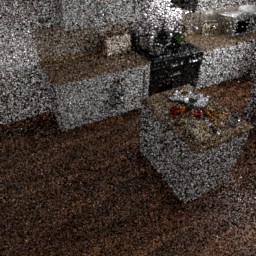} &
		\staticimagegroupkitchen{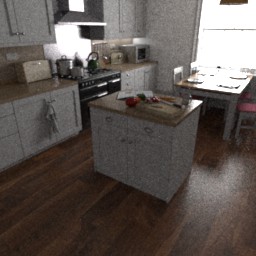} \\
		\footnotesize \makecell{Teaching Input} &
		\footnotesize \makecell{Prediction} &
		\footnotesize \makecell{Ground Truth} &
		\footnotesize \makecell{Teaching Input} &
		\footnotesize \makecell{Prediction} &
		\footnotesize \makecell{Ground Truth} \\
	\end{tabular}
	\caption{
		\textbf{Kitchen Scene.}
		Results of our model on a realistic static kitchen scene.
		\label{fig:results-static-kitchen}
	}
\end{figure}

For the static scenes ablation study in \secref{sec:ablation-static} we tested different input configurations for our network, showing that we achieve the best possible outcome by combining all of the inputs.
Additional visual results for this experiment are shown in \figref{fig:ablation-static-supp}.

We also tested our model on two additional challenging static scenes, with results shown in \figref{fig:results-static-bathroom} and \figref{fig:results-static-kitchen}, respectively.
For this experiment, we used a realistic bathroom scene and a realistic kitchen scene~\cite{Bittlerli2016} at an image resolution of $256\times 256$ pixels.
Both scenes were trained with a batch size of 128 for 150,000 iterations.
Although there is no light transport to be learned in these static scene experiments, we find that our model is able to encode realistic static scenes well, and renders novel views accurately.

For dynamic scenes we also conducted two experiments:
one where we compared our approach to a set of baselines in \secref{sec:dynamic-objects-fixed-lights}, using a dataset with dynamic objects and fixed lights, where we translated and removed objects randomly.
And another experiment where we highlight the importance of the Light Transport Layer and the additional photon architecture in \secref{sec:dynamic-objects-dynamic-lights}, on a dataset with dynamic objects and dynamic lights, where we additionally translate rectangular light sources randomly along the ceiling.
\tabref{tab:baselines1} shows the full quantitative evaluation of the experiments for dynamic objects and fixed lights. 
For the experiment with dynamic objects and dynamic lights we provide a full quantitative evaluation in~\tabref{tab:baselines2}.

\figref{fig:results-baselines} shows additional visual results for the baseline comparison for dynamic objects and fixed lights, complementing \figref{fig:dynamic-scenes}.
In \figref{fig:error-comparison} we show examples where our method does not predict the illumination accurately.
These error images also show that for the denoising approaches errors occur mostly in image regions with high frequency components, \ie edges and textures.
For our approach, errors sometimes also occur in larger image regions when the prediction is inaccurate or sparse.
This also explains that while our approach performs best for most of the metrics in~\tabref{tab:baselines1} and~\tabref{tab:baselines2}, the MSE is lower for the denoising approaches.

In addition to the results shown in \figref{fig:photon-importance}, we show visual results and error images for dynamic objects and dynamic lights in \figref{fig:results-photon-diff}, as well as failure cases in \figref{fig:failures}.

In addition to the quantitative comparison for dynamic objects and fixed lights (\figref{fig:denoising-graph}), we show a more comprehensive quantitative comparison for dynamic objects and dynamic lights in~\figref{fig:denoisingd2}.
In addition to MSSIM and FID, we compare L1 feature losses from different stages of the Inception v3 network \cite{Szegedy2015CVPR, Szegedy2016CVPR}, showing that our approach clearly outperforms the denoising baselines on different levels of image abstraction.

\clearpage

\begin{table}[t!]
	\centering
	\setlength{\tabcolsep}{17pt}
	\begin{footnotesize}
		\begin{tabular*}{\linewidth}{lrrrrr}
			\toprule
			Architecture &   time / frame &    MSE &     MSSIM &         FID &  Feature L1 \\
			\midrule
			Denoising (1/1) &  1.5059s &  0.0005 &  0.880 &   26.4 &                 0.163 \\
			Denoising (1/4) &  0.3800s &  0.0007 &  0.867 &   28.1 &                 0.172 \\
			Denoising (1/16) &  0.0986s &  0.0012 &  0.835 &   38.7 &                 0.203 \\
			Denoising (1/32) &  0.0532s &  0.0018 &  0.813 &   54.1 &                 0.233 \\
			Denoising (1/64) &  0.0283s &  0.0029 &  0.781 &   94.0 &                 0.281 \\
			\midrule
			CNN only &  0.0191s &  0.0043 &  0.835 &   36.1 &                 0.195 \\
			Feature projection &  0.0210s &  0.0037 &  0.841 &   32.5 &                 0.185 \\
			Ours (w/o photons) &  0.0243s &  0.0044 &  0.841 &   31.4 &                 0.184 \\
			Ours (w/ photons) &  0.0459s &  0.0028 &  0.849 & 30.6 &  0.182 \\
			\bottomrule
		\end{tabular*}
	\end{footnotesize}
	\vspace{0.3cm}
	\caption{
		\textbf{Dynamic Objects and Fixed Lights.}
		Quantitative evaluation for our experiment on dynamic objects and fixed lights.
	}
	\label{tab:baselines1}
\end{table}

\begin{table}[t!]
	\centering
	\setlength{\tabcolsep}{17pt}
	\begin{footnotesize}
		\begin{tabular*}{\linewidth}{lrrrrr}
			\toprule
			Architecture &     time / frame &     MSE &  MSSIM &    FID & Feature L1 \\
			\midrule
			Denoising (1/1) &  1.5059s &  0.0002 &  0.930 &   17.1 &                 0.137 \\
			Denoising (1/4) &  0.3801s &  0.0002 &  0.923 &   17.6 &                 0.143 \\
			Denoising (1/16) &  0.0988s &  0.0005 &  0.896 &   23.5 &                 0.172 \\
			Denoising (1/32) &  0.0518s &  0.0008 &  0.874 &   38.6 &                 0.207 \\
			Denoising (1/64) &  0.0283s &  0.0016 &  0.839 &   84.1 &                 0.269 \\
			\midrule
			CNN only &  0.0190s &  0.0100 &  0.827 &   33.7 &                 0.199 \\
			Feature projection &  0.0208s &  0.0098 &  0.827 &   32.9 &                 0.197 \\
			Ours (w/o photons) &  0.0243s &  0.0029 &  0.871 &   30.0 &                 0.184 \\
			Ours (w/ photons) &  0.0468s &  0.0014 &  0.887 &   25.1 &  0.172 \\
			\bottomrule
		\end{tabular*}
	\end{footnotesize}
	\vspace{0.3cm}
	\caption{
		\textbf{Dynamic Objects and Dynamic Lights.}
		Quantitative evaluation for our experiment on dynamic objects and dynamic lights.
	}
	\label{tab:baselines2}
\end{table}

\begin{figure}[p]
	\centering
	\includegraphics[width=0.45\linewidth]{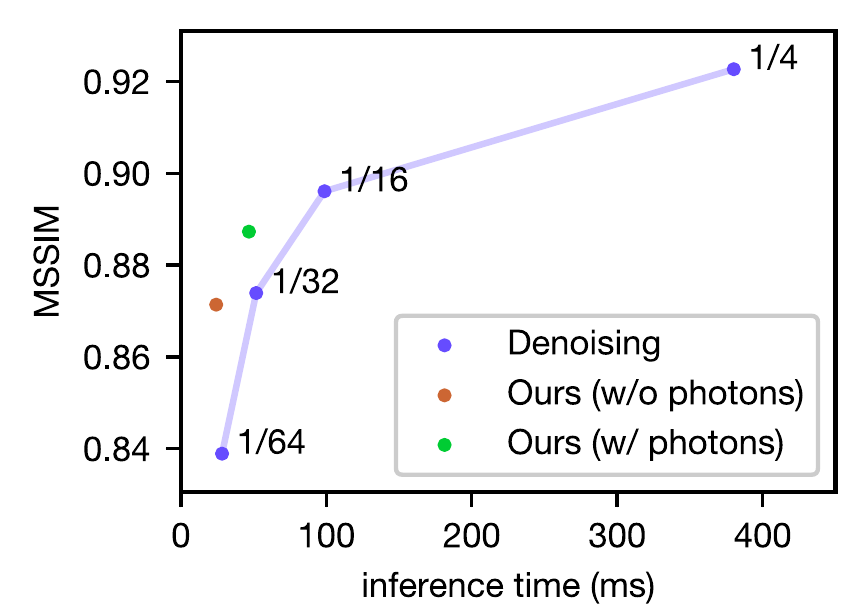}\hspace{0.5cm}%
	\includegraphics[width=0.45\linewidth]{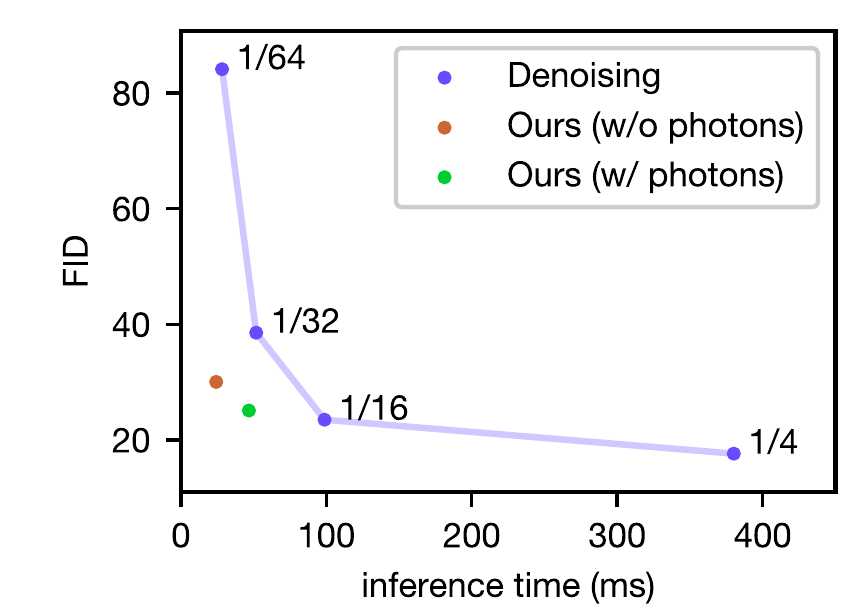}\\\vspace{0.5cm}%
	\includegraphics[width=0.45\linewidth]{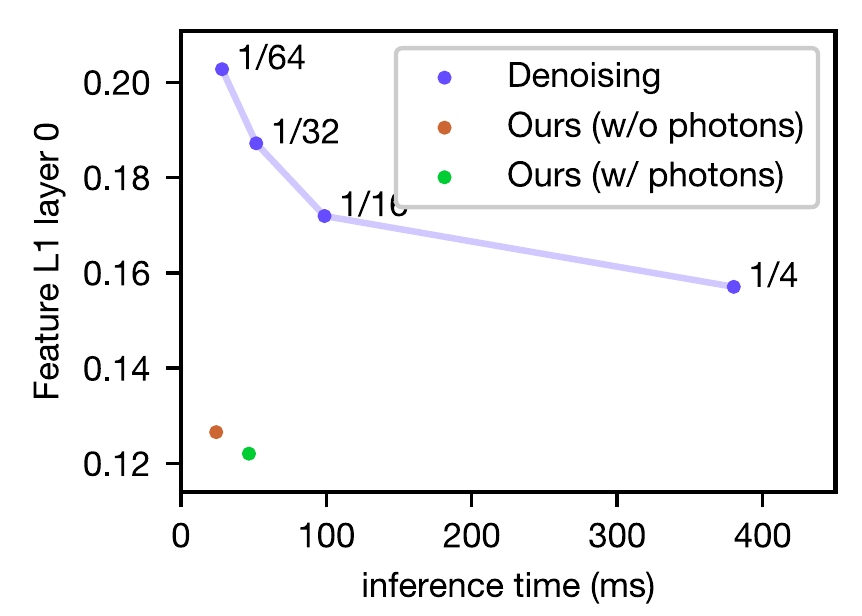}\hspace{0.5cm}%
	\includegraphics[width=0.45\linewidth]{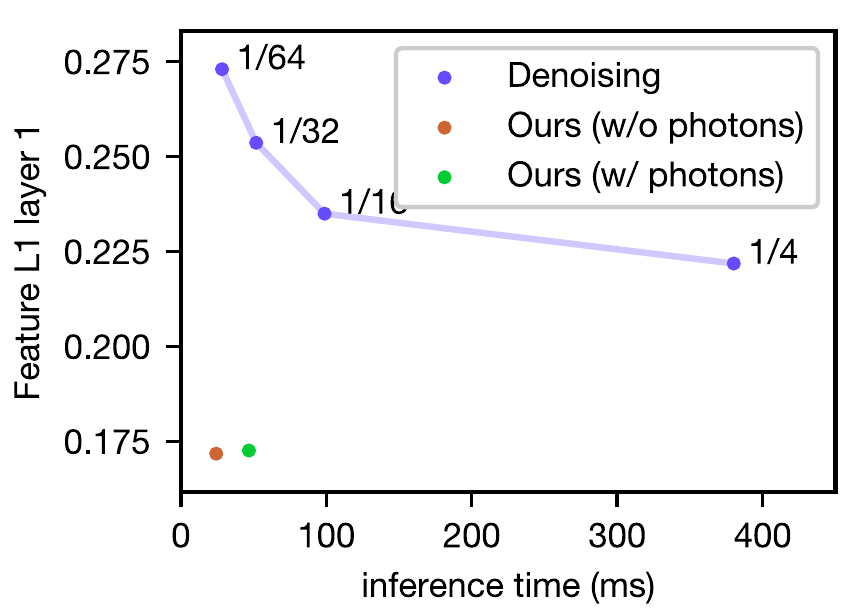}\\\vspace{0.5cm}%
	\includegraphics[width=0.45\linewidth]{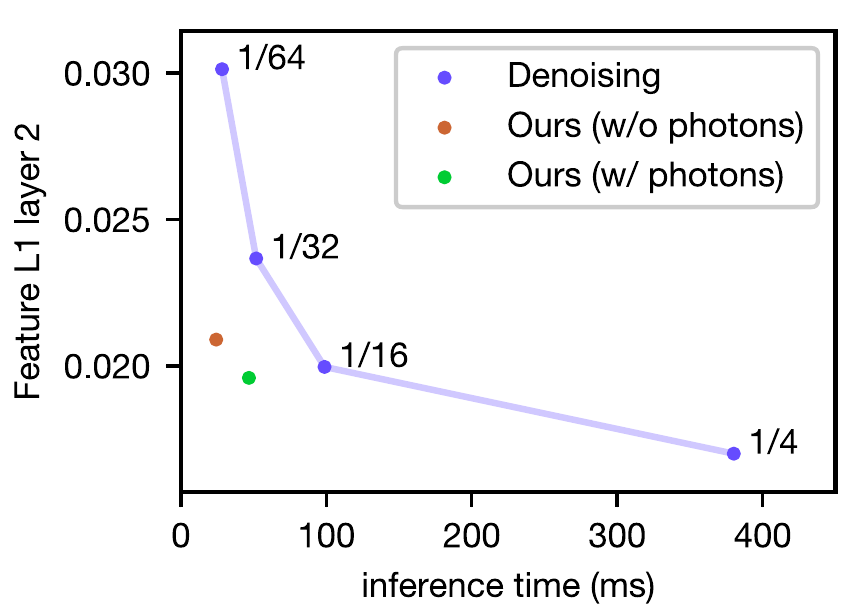}\hspace{0.5cm}%
	\includegraphics[width=0.45\linewidth]{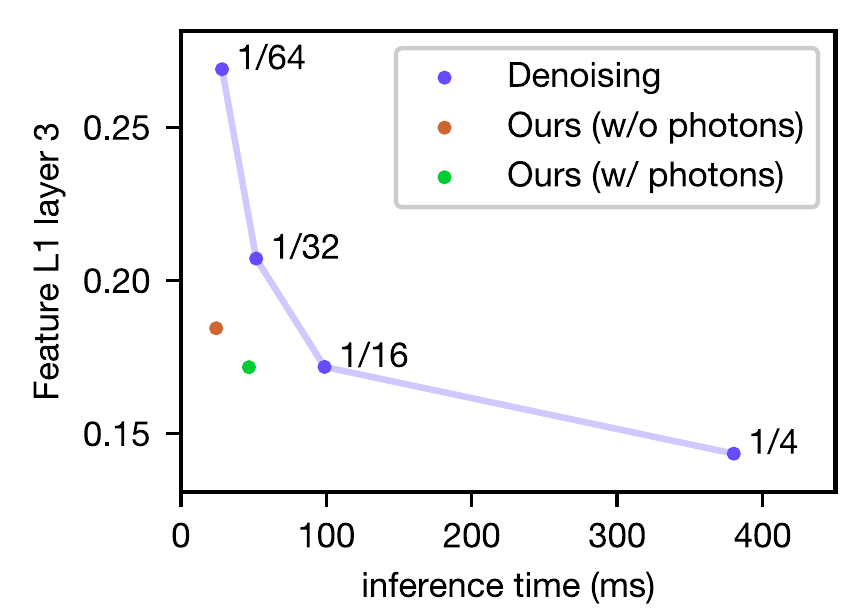}%
	\caption{\textbf{Dynamic Objects and Dynamic Lights.}
		This plot shows a quantitative comparison of our approach with the denoising baseline for different sample densities.
		We plot reconstruction accuracy over inference time for our experiment on dynamic objects and dynamic lights.
		The denoising labels refer to the ratio of pixels that are dropped.
		The layer indices (0--3) for the Feature L1 losses refer to outputs of the four major layers in the Inception v3 network.
	}
	\label{fig:denoisingd2}
\end{figure}

\def\w{0.158\linewidth}

\def\imagerow#1#2{
	\includegraphics[width=.12\linewidth]{#1/sv/#2} &
	\includegraphics[width=.12\linewidth]{#1/denoise/#2} &
	\includegraphics[width=.12\linewidth]{#1/denoise64/#2} &
	\includegraphics[width=.12\linewidth]{#1/cnn/#2} &
	\includegraphics[width=.12\linewidth]{#1/projection/#2} &
	\includegraphics[width=.12\linewidth]{#1/pointnet/#2} &
	\includegraphics[width=.12\linewidth]{#1/photon/#2} &
	\includegraphics[width=.12\linewidth]{#1/gt/#2} \\
}

\def\imagerowdiff#1#2{
	\includegraphics[width=.118\linewidth]{#1/projection/#2} &
	\includegraphics[width=.118\linewidth]{#1/projection_diff/#2} &
	\includegraphics[width=.118\linewidth]{#1/pointnet/#2} &
	\includegraphics[width=.118\linewidth]{#1/pointnet_diff/#2} &
	\includegraphics[width=.118\linewidth]{#1/photon/#2} &
	\includegraphics[width=.118\linewidth]{#1/photon_diff/#2} &
	\includegraphics[width=.118\linewidth]{#1/gt/#2} \\
}

\begin{figure}[p]
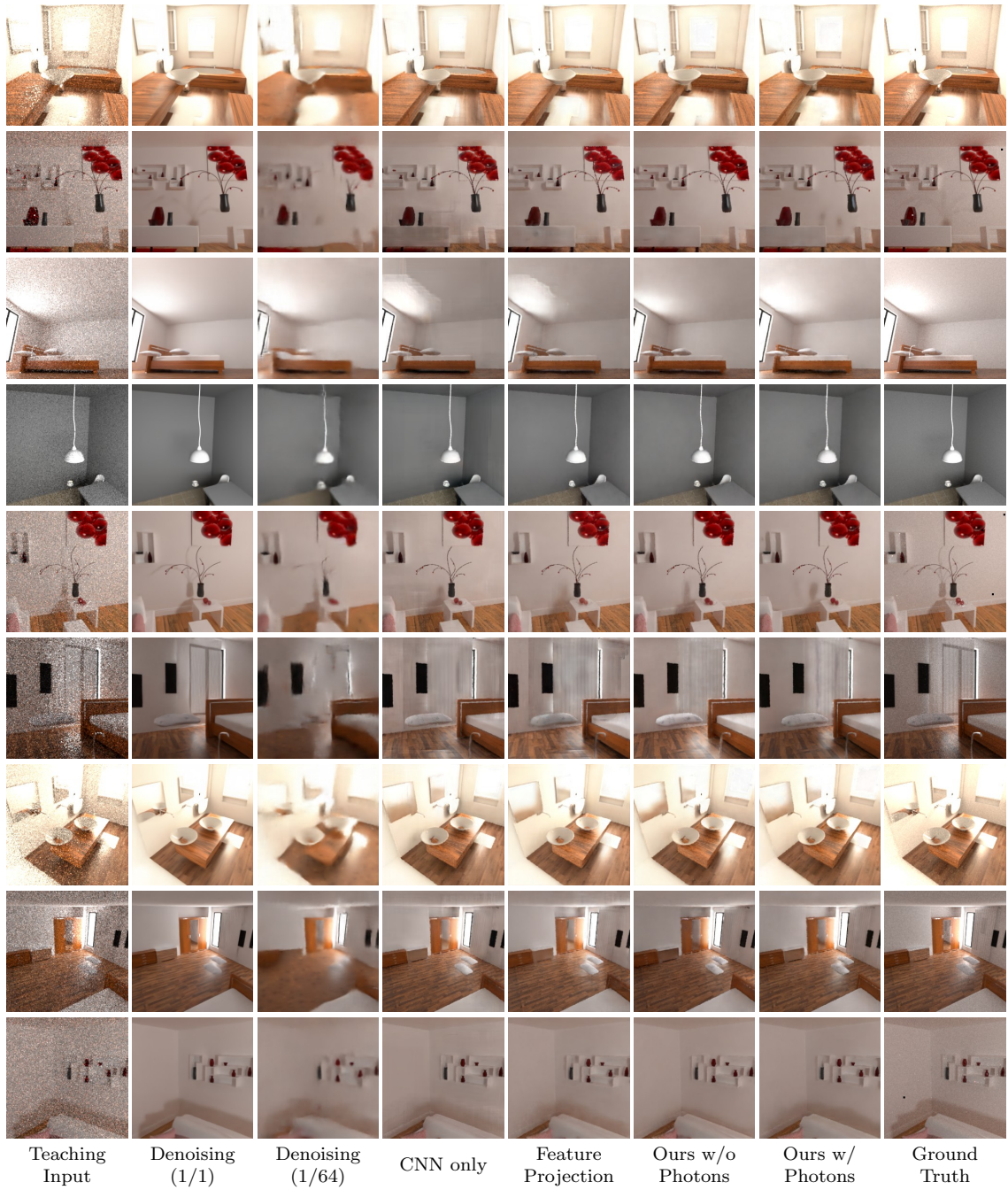

    \centering
    \setlength\tabcolsep{1pt} %
	\renewcommand{\arraystretch}{.7}
	\begin{tabular}{cccccccc}
		\imagerow{compressed/d1_256/it00300000}{00020.jpg}
		\imagerow{compressed/d1_256/it00300000}{00085.jpg}
		\imagerow{compressed/d1_256/it00300000}{00075.jpg}
		\imagerow{compressed/d1_256/it00300000}{00034.jpg}
		\imagerow{compressed/d1_256/it00300000}{00009.jpg}
		\imagerow{compressed/d1_256/it00300000}{00015.jpg}
		\imagerow{compressed/d1_256/it00300000}{00016.jpg}
		\imagerow{compressed/d1_256/it00300000}{00019.jpg}
		\imagerow{compressed/d1_256/it00300000}{00025.jpg}
		\footnotesize \makecell{Teaching\\ Input} &
		\footnotesize \makecell{Denoising\\  (1/1)} & 
		\footnotesize \makecell{Denoising\\  (1/64)}  & 
		\footnotesize \makecell{CNN only} & 
		\footnotesize \makecell{Feature \\ Projection} & 
		\footnotesize \makecell{Ours w/o\\ Photons} & 
		\footnotesize \makecell{Ours w/\\ Photons} &
		\footnotesize \makecell{Ground\\ Truth} \\
	\end{tabular}
    \caption{\textbf{Dynamic Objects and Fixed Lights.}
        Additional results for our method as well as for the baselines for dynamic objects and fixed lights, complementing \figref{fig:dynamic-scenes}. 
    }
    \label{fig:results-baselines}
\end{figure}

\def\imagerowdiffcompare#1#2{
	\includegraphics[width=\w]{#1/denoise4/#2} &
	\includegraphics[width=\w]{#1/denoise16/#2} &
	\includegraphics[width=\w]{#1/denoise32/#2} &
	\includegraphics[width=\w]{#1/denoise64/#2} &
	\includegraphics[width=\w]{#1/pointnet/#2} &
	\includegraphics[width=\w]{#1/photon/#2} \\
	\includegraphics[width=\w]{#1/denoise4_diff/#2} &
	\includegraphics[width=\w]{#1/denoise16_diff/#2} &
	\includegraphics[width=\w]{#1/denoise32_diff/#2} &
	\includegraphics[width=\w]{#1/denoise64_diff/#2} &
	\includegraphics[width=\w]{#1/pointnet_diff/#2} &
	\includegraphics[width=\w]{#1/photon_diff/#2} 
}

\begin{figure}[p]
    \centering
    \small
	\setlength\tabcolsep{1pt} %
		\begin{tabular}{cccccc}
		\imagerowdiffcompare{compressed/d1_256/it00300000}{00018.jpg}\\[0.5em]
		\imagerowdiffcompare{compressed/d1_256/it00300000}{00012.jpg}\\[0.5em]
		\imagerowdiffcompare{compressed/d1_256/it00300000}{00019.jpg}\\
		\footnotesize \makecell{Denoising (1/4)} & 
		\footnotesize \makecell{Denoising (1/16)} & 
		\footnotesize \makecell{Denoising (1/32)} & 
		\footnotesize \makecell{Denoising (1/64)} & 
		\footnotesize \makecell{Ours w/o photons} & 
		\footnotesize \makecell{Ours w/ photons} \\
	\end{tabular}
	\vspace{0.3em}
    \caption{\textbf{Dynamic Objects and Fixed Lights.}
        Predictions and error images with respect to ground truth for different denoising approaches and our approach for dynamic objects and fixed lights.
        Error plots are shown below the respective prediction.
    }
    \label{fig:error-comparison}
\end{figure}

\begin{figure}[p]
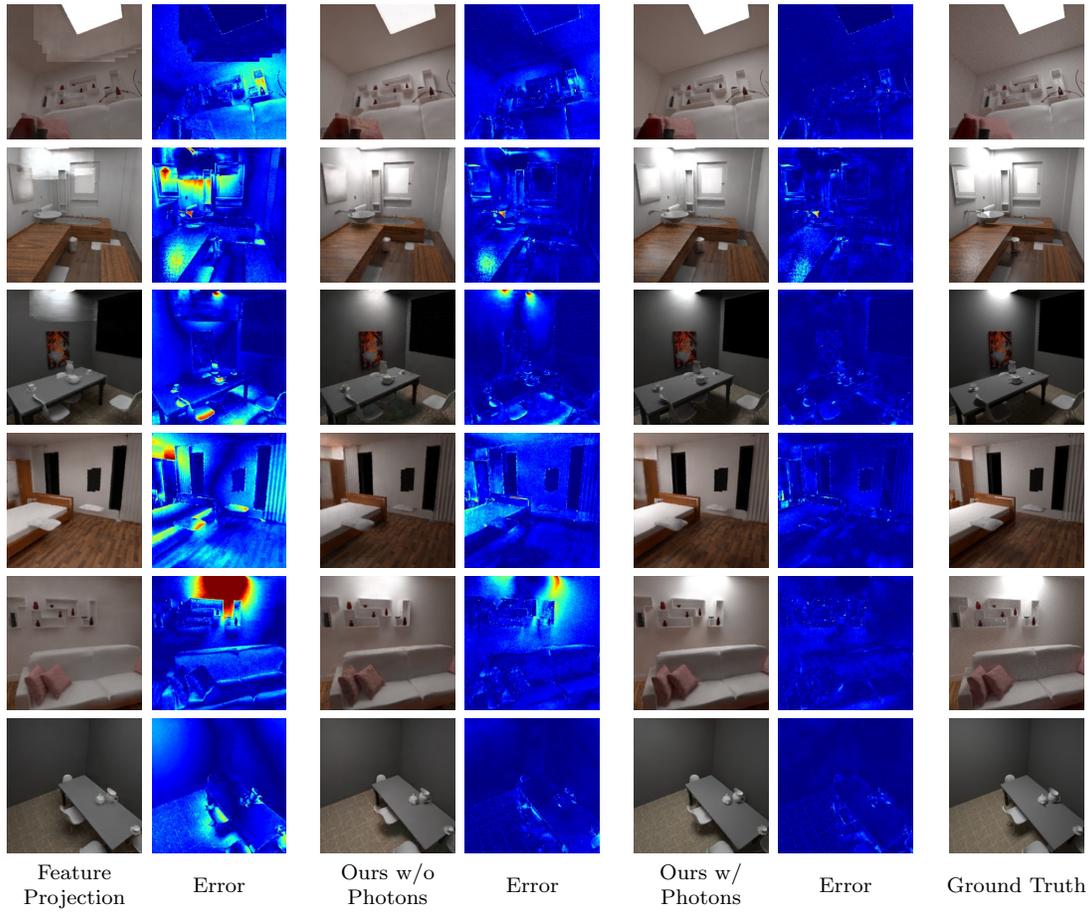

    \centering
    \small
	\setlength\tabcolsep{2pt} %
		\begin{tabular}{cc>{\hspace{9pt}}cc>{\hspace{9pt}}cc>{\hspace{9pt}}c}
		\imagerowdiff{compressed/d2_256/it00300000}{00001.jpg}
		\imagerowdiff{compressed/d2_256/it00300000}{00020.jpg}
		\imagerowdiff{compressed/d2_256/it00300000}{00110.jpg}
		\imagerowdiff{compressed/d2_256/it00300000}{00023.jpg}
		\imagerowdiff{compressed/d2_256/it00300000}{00045.jpg}
		\imagerowdiff{compressed/d2_256/it00300000}{00018.jpg}

		\footnotesize \makecell{Feature \\ Projection} & 
		\footnotesize Error &
		\footnotesize \makecell{Ours w/o\\ Photons} & 
		\footnotesize Error &
		\footnotesize \makecell{Ours w/\\ Photons} &
		\footnotesize Error &
		\footnotesize Ground Truth
	\end{tabular}
	\vspace{0.3em}
    \caption{\textbf{Dynamic Objects and Dynamic Lights.}
        Additional visual results for dynamic objects and dynamic lights, complementing \figref{fig:photon-importance}.
    }
    \label{fig:results-photon-diff}
\end{figure}

\begin{figure*}[p]
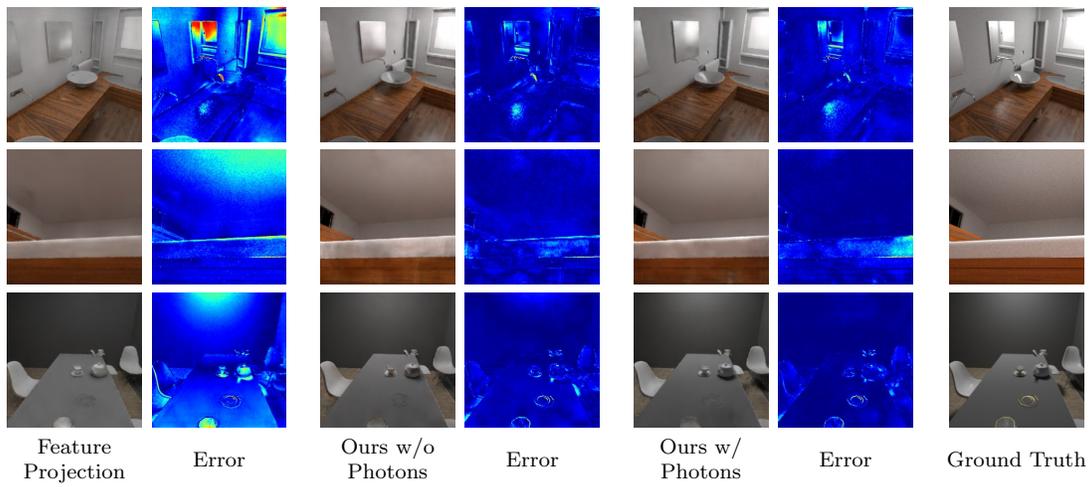

    \centering
    \small
	\setlength\tabcolsep{2pt} %
		\begin{tabular}{cc>{\hspace{9pt}}cc>{\hspace{9pt}}cc>{\hspace{9pt}}c}
		\imagerowdiff{compressed/d2_256/it00300000}{00088.jpg}
		\imagerowdiff{compressed/d2_256/it00300000}{00075.jpg}
		\imagerowdiff{compressed/d2_256/it00300000}{00122.jpg}
		\footnotesize \makecell{Feature \\ Projection} & 
		\footnotesize Error &
		\footnotesize \makecell{Ours w/o\\ Photons} & 
		\footnotesize Error &
		\footnotesize \makecell{Ours w/\\ Photons} &
		\footnotesize Error &
		\footnotesize Ground Truth
	\end{tabular}
	\vspace{0.3em}
    \caption{\textbf{Dynamic Objects and Dynamic Lights.}
    	Example scenarios that are challenging for our approach with dynamic objects and dynamic lights.
        We observe failure cases for specular materials and mirrors, when objects are close to the camera and in the presence of fine shadows.
    }
    \label{fig:failures}
\end{figure*}

\clearpage